\documentclass{article}
\usepackage{subcaption}
\usepackage{multirow}
\usepackage{algorithm}

 \usepackage[main, final]{neurips_2026}
\usepackage[utf8]{inputenc} % allow utf-8 input
\usepackage[T1]{fontenc}    % use 8-bit T1 fonts
\usepackage{hyperref}       % hyperlinks
\usepackage{url}            % simple URL typesetting
\usepackage{booktabs}       % professional-quality tables
\usepackage{amsfonts}       % blackboard math symbols
\usepackage{nicefrac}       % compact symbols for 1/2, etc.
\usepackage{microtype}      % microtypography
\usepackage{xcolor}         % colors

\usepackage{hyperref}
\usepackage{algpseudocode}

\usepackage{amsmath}
\usepackage{amssymb}
\usepackage{mathtools}
\usepackage{amsthm}

\usepackage[capitalize,noabbrev]{cleveref}

\theoremstyle{plain}
\newtheorem{theorem}{Theorem}[section]
\newtheorem{proposition}[theorem]{Proposition}

\theoremstyle{definition}

\theoremstyle{remark}

\title{METRO: Metric-Enhanced Token Routing Operator}

\author{%
  Nodens Koren \\
  Department of Computer Science\\
  ETH Zürich\\
  Zürich, Switzerland \\
  \texttt{nodens.koren@inf.ethz.ch} \\
  \And
  Thomas Hofmann \\
  Department of Computer Science\\
  ETH Zürich\\
  Zürich, Switzerland \\
  \texttt{thomas.hofmann@inf.ethz.ch} \\
  \AND
  Georgios Kissas \\
  Swiss Data Science Center \\
  ETH Zürich \\
  Zürich, Switzerland \\
  \texttt{georgios.kissas@sdsc.ethz.ch} \\
}

\begin{document}

\maketitle

\begin{abstract}
State-of-the-art neural operators scale to complex meshes via \textit{slice-and-process} architectures, yet many rely on linear compatibility scores for latent tokenization. Under common feature normalization, such scores are equivalent to isotropic Euclidean clustering, while without normalization they induce unbounded linear decision regions. In both cases, they lack slice-specific anisotropic locality, which can lead to redundant and entangled latent slices. To address this, we propose \textbf{Metric-Enhanced Token Routing Operator (METRO)}, a geometry-aware routing mechanism that replaces linear projection with a learnable Mahalanobis metric. By enabling each latent slice to learn a local anisotropic tensor, METRO shapes receptive fields into exponentially localized, oriented ellipsoids that naturally align with flow features like boundary layers and wakes. As a drop-in replacement, METRO yields consistent improvements across both Transformer and Mamba backbones. Empirically, our method achieves substantial performance gains on irregular domains, outperforming baselines on both standard PDE benchmarks and complex industrial design tasks. Finally, METRO exhibits enhanced robustness in out-of-distribution regimes across varying Reynolds numbers and geometric configurations. \looseness=-1
\end{abstract}

\section{Introduction}

Industrial design workflows in sectors such as aerospace, automotive, energy, and turbomachinery rely on solving Partial Differential Equations (PDEs) repeatedly for varying system configurations. While traditional solvers offer high accuracy, they are computationally  slow, making them impractical for rapid prototyping, uncertainty quantification, and real-time control \citep{lu2021learning, lifourier}. In light of these bottlenecks, deep learning algorithms (surrogate models) have recently been explored to predict the behavior of physical systems, with a specific focus on neural operators that learn the solution operator of PDEs directly. 

Neural operators \citep{kovachki2023neural, li2023fourier, azizzadenesheli2024neural, kissas2022learning, seidman2022nomad, lingsch2024fuse, lu2021learning} have evolved from solving proof-of-concept problems to serving as accurate, scalable surrogates increasingly adopted in complex engineering workflows, from climate modeling to structural design~\citep{bonev2023spherical, wu2024transolver, azizzadenesheli2024neural, tiwarilatent}. To achieve broader adoption in large-scale industrial applications, such as aerodynamic shape optimization, structural analysis, and turbomachinery, they need to: (i) process complex geometries without relying on naive interpolation to regular grids, which disregards the underlying geometric structure and the physical nature of the problem; (ii) scale to a massive number of degrees of freedom to accommodate the high-resolution discretizations required for industrial fidelity; (iii) capture long-range interactions, as complex physical phenomena like wakes often span the entire domain; and (iv) capture anisotropic information propagation, where features such as boundary layers and shocks are thin in one direction yet distinctively elongated in another. 

A common strategy to address these challenges is to \emph{compress} the input field to a much smaller set of latent tokens, perform global mixing in this latent space, and \emph{decode} back to the original resolution. Existing methods differ mainly in how the mesh-to-latent coupling is defined (e.g., local neighborhood integration \citep{wen2025geometry, li2023geometry} versus global soft routing \citep{hu2025transolver, luotransolver++}). One prominent approach in this compress-process-decode family is the so-called ``slice-and-process'' method.  Despite their success, current slice-based operators rely on simplified geometric assumptions that limit their expressivity. For instance, the original Transolver employs dot-product attention, which defines linear decision boundaries. As we derive in Section~\ref{sec:isotropy}, under standard feature normalization, dot-product routing is equivalent, up to constants, to nearest-prototype assignment under an isotropic Euclidean distance: it clusters points based on Euclidean proximity to a prototype, effectively favoring spherical neighborhoods in the learned feature space. While subsequent works~\citep{luotransolver++} have improved computational efficiency, this limitation in anisotropic local tokenization has yet to be addressed. Consequently, there remains a critical gap for a mechanism that is \textit{locally adaptive}, allowing different latent slices to learn distinct anisotropic metric tensors in the learned feature space.

Motivated by these insights, we propose \textbf{Metric-Enhanced Token Routing Operator (METRO)}, a principled generalization of the latent-attention paradigm. METRO endows neural operators with \textit{learnable anisotropic metrics} that shape each attention head’s receptive field into an exponentially localized ellipsoid, rather than an unbounded hyperplane or a rigid sphere. Specifically, we replace the standard isotropic dot-product in the attention routing with a Mahalanobis distance tailored to the underlying data. This yields \textit{geometry-aware attention}: each latent slice learns its own local metric tensor, allowing it to attend to points along specific orientations while suppressing off-axis or outlier features. The metric is learned end-to-end, providing the model with the flexibility to discover the optimal anisotropy for each physical region.

In summary, our key contributions are:
\begin{itemize}
    \item Task-Specialized Local Routing: We introduce a Mahalanobis-based routing mechanism where each latent slice learns its own distinct anisotropic metric. Unlike global feature reweighting approaches, this enables the model to create a structured latent partition where different slices can independently specialize to capture varying physical modes—such as highly anisotropic boundary layers or isotropic diffusion—within the same system.
    \item High-Efficiency Anisotropic Tokenization: By effectively shaping the receptive fields to match complex geometries, our method captures physical phenomena using significantly fewer latent slices. Under strict parameter-matched conditions, this compact representation yields both lower error rates and faster training times compared to isotropic baselines.
    \item Model-Agnostic Framework: Our approach is a modular upgrade compatible with both Transformer and state-space model (SSM) backbones, providing richer geometric awareness without sacrificing linear scalability. \looseness=-1
\end{itemize}

Beyond these technical contributions, METRO aims to broaden the applicability of latent operator networks. Our empirical results highlight substantial accuracy gains across diverse PDE systems, a nearly $50\%$ reduction in volumetric prediction error on the AirfRANS \citep{bonnet2022airfrans} benchmark, and superior generalization in aerodynamic design tasks.

\section{Literature review}

\textbf{Neural operators and PDE solvers.} 
Surrogate models for physics simulations on irregular domains can be broadly categorized based on their strategy for handling geometric discretization. \textit{Transformation-based approaches}, including Geo-FNO~\citep{li2023fourier}, Diffeomorphic Neural Operators~\citep{ahmad2024diffeomorphic}, Optimal Transport Neural Operators~\citep{li2025geometric}, and DIMON~\citep{yin2024dimon}, map non-uniform meshes into regular latent grids to leverage scalable backbones such as FNO \citep{lifourier}. However, these methods rely on diffeomorphisms that are often too topologically restrictive for complex geometries, or on optimal transport maps that may yield non-smooth, non-invertible coordinates across varying instances. \textit{Graph-based methods}, such as GNO~\citep{li2020neural}, GINO~\citep{li2023geometry}, and MeshGraphNet~\citep{pfaff2020learning}, instead utilize graph embeddings to process irregular meshes directly or map them to global representations. Despite their utility, these approaches typically incur high computational costs due to multiple message-passing rounds, the need for multi-scale receptive fields, and the hardware inefficiencies associated with unstructured memory access. \textit{Direct methods}, such as GNOT~\citep{hao2023gnot} and arbitrary-grid FNOs~\citep{lingsch2023beyond}, operate on the physical domain without intermediate mappings. While mathematically robust, operating on the full physical discretization often hinders efficient scaling to high-resolution three-dimensional problems. Finally, \textit{embedding approaches}—such as the slicing mechanism in Transolver~\citep{wu2024transolver, luotransolver++, hu2025transolver} or continuous neural representations~\citep{serrano2023operator, yin2022continuous}—offer high flexibility and scalability by treating geometry as unstructured tokens. While they bypass rigid topological constraints, existing methods fail to explicitly preserve intrinsic geometric properties.

\begin{figure*}[t]
    \centering % This centers the whole block
    
    % --- LEFT PANEL: Flow Chart ---
    \begin{subfigure}[b]{0.3\textwidth}
        \centering
        \includegraphics[width=\linewidth, keepaspectratio]{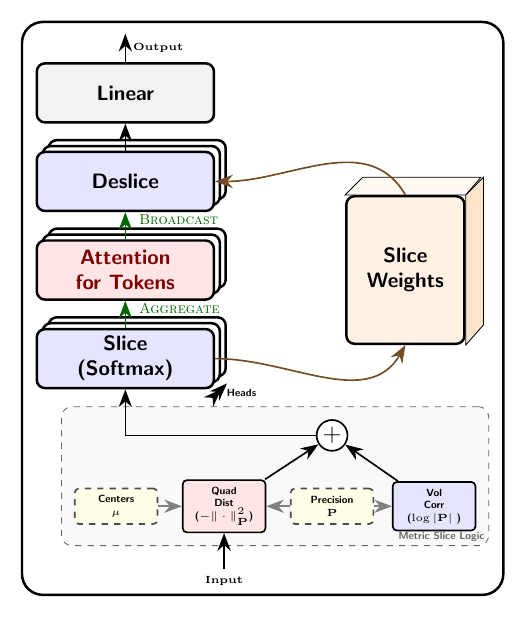}
        \caption{Pipeline}
        \label{fig:malo_flowchart}
    \end{subfigure}%
    \hspace{0.5cm} % Fixed gap
    % --- RIGHT PANEL: Visualization ---
    \begin{subfigure}[b]{0.5\textwidth}
        \centering
        \includegraphics[angle=270, width=\linewidth, keepaspectratio]{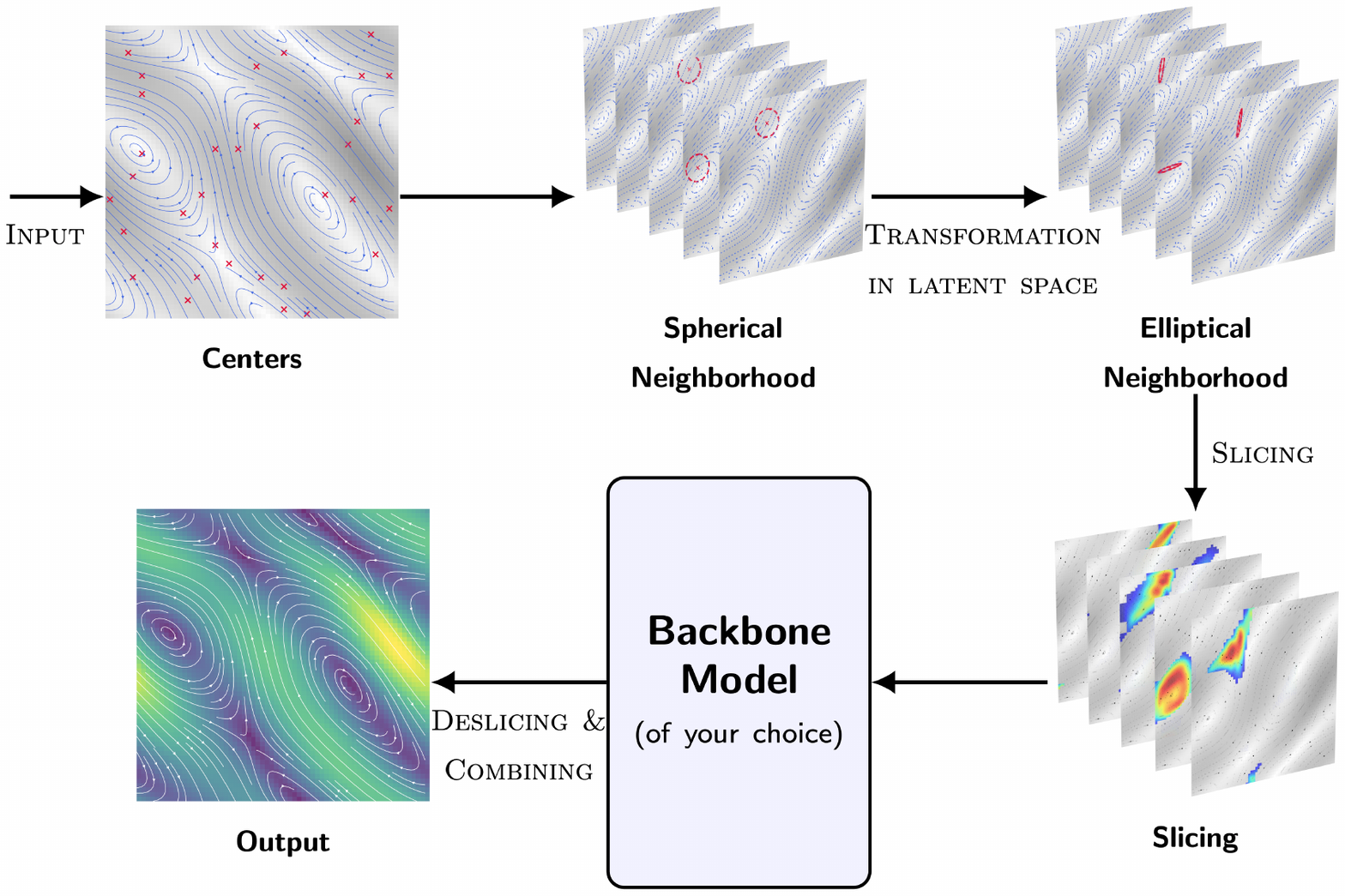}
        % UPDATED CAPTION HERE:
        \caption{Workflow Visualization}
        \label{fig:malo_viz}
    \end{subfigure}

    \caption{\textbf{Overview of the METRO Architecture.} 
    (\textbf{a}) The processing pipeline: Input features are mapped to centers, transformed by the metric, sliced, and processed.
    (\textbf{b}) Visual demonstration: Comparing standard spherical neighborhoods (Iso) to flow-aligned elliptical neighborhoods (Aniso), showing how the model explicitly attends to local physical structures.}
    \label{fig:main_system}
\end{figure*}

\section{Methodology}
\label{sec:method}
In this section, we formalize latent operator learning, analyze the geometric limitations of previously proposed approaches, and introduce our core contribution: \textit{Metric-Enhanced Token Routing Operator (METRO)}. More specifically, we demonstrate that linear routing in slicing operations induces an isotropic partition of the latent space under the standard normalization assumptions. To address this, we propose a learnable anisotropic metric. We define the resulting geometry by interpreting the learned latent metric as inducing a pullback pseudometric on the physical domain, a formulation that simplifies computation.

\subsection{Problem Setup and Notation}

Neural operators learn an approximate map between infinite-dimensional  function spaces $\mathcal{A} \to \mathcal{U}$. We consider a physical domain $\Omega$ discretized by a (non-)uniform mesh with $N$ points, and an input embedding $\mathbf {x}_n \in \mathbb{R}^D$ at point $n$, obtained from an encoder $\phi: \Omega \times \mathbb{R}^{d_{in}} \rightarrow \mathcal{Z} \subset \mathbb{R}^D$ acting on an input field and geometry. Crucially, all routing operations are performed in this learned feature space rather than directly in physical coordinates. The ``Slice-and-Process'' paradigm \citep{wu2024transolver} compresses these $N$ points into $G$ latent tokens (slices), where $G \ll N$. These tokens are processed via a neural backbone and then projected back to the original resolution:
\begin{align*}
    &\text{Slicing} \qquad \mathbf{z}_g = \frac{\sum_{n=1}^NA_{n,g} \mathbf{x}_n }{\sum_{n=1}^N A_{n,g}} \in \mathbb{R}^D, \quad g = 1, ..., G \\
    &\text{Process} \qquad (\mathbf{z}'_1, ..., \mathbf{z}'_G) = \mathcal{G} (\mathbf{z}_1, ..., \mathbf{z}_G), \\
    &\text{de-Slicing} \qquad \mathbf{x}'_n = \sum_{g=1}^G A_{n,g} \mathbf{z}'_g \in \mathbb{R}^D, \quad n = 1,...,N. 
\end{align*}
Here $A_{n,g} \ge 0$ represents the routing weights, satisfying $\sum_g A_{n,g} = 1$ for each $n$, and $\mathcal{G}$ denotes an arbitrary backbone, such as Transformer \citep{hu2025transolver, luotransolver++} or Mamba \citep{tiwarilatent}. METRO modifies the routing mechanism as detailed in Section \ref{sec:malo_slicing}. 

\subsection{Revisiting Linear Slicing: The Isotropy Equivalence}
\label{sec:isotropy}
\citet{wu2024transolver} proposed a slicing mechanism where routing weights are derived from a scalar compatibility score $S_{n,g}$ via a Softmax operator:
\begin{equation}
    \label{eq:generic_routing}
    A_{n,g} = \text{Softmax}_g \left( S_{n,g} \right),
\end{equation}
In the original Transolver, this score is a dot product $S_{n,g} = \frac{\mathbf{x}_n^\top \mathbf{w}_g}{\tau}$, where $\tau > 0$ is the temperature, utilizing learnable prototypes $\mathbf{w}_g \in \mathbb{R}^D$. This scoring induces linear pairwise assignment boundaries. Under the normalized/equal-norm interpretation below, these boundaries correspond to an isotropic Euclidean Voronoi partition.

Under standard feature normalization—such as LayerNorm \citep{ba2016layer} or an explicit $\ell_2$ normalization—maximizing this dot-product is equivalent, up to a constant, to minimizing the isotropic Euclidean distance: $\|\mathbf{x} - \mathbf{w}\|_2^2 = \|\mathbf{x}\|_2^2 + \|\mathbf{w}\|_2^2 - 2 \mathbf{x}^\top \mathbf{w}$. Given the normalization regime where $\| \mathbf{x}  \|_2^2 \approx \text{const}$ and $\| \mathbf{w}  \|_2^2 \approx \text{const}$, using the score $\mathbf{x}^\top \mathbf{w}$ is functionally equivalent to using the distance $ \| \mathbf{x} - \mathbf{w} \|_2^2$.

This reveals two distinct views of the original routing mechanism. Without feature normalization, the linear dot-product score induces linear decision boundaries (i.e., unbounded hyperplanes). Under standard normalization, maximizing the dot-product becomes equivalent to minimizing an isotropic Euclidean distance, leading to spherical clustering. What is common to both views is that linear routing does not induce a localized, anisotropic interior region. In contrast, the proposed Mahalanobis score defines ellipsoidal level sets with exponential locality. Since information flow in physical systems is often highly preferential—such as in boundary layers or wakes—this motivates the need for an anisotropic routing geometry.

\subsection{Geometry-Aware Anisotropic Slicing}
\label{sec:malo_slicing}

To induce anisotropic similarity within the slice-process-deslice paradigm, we replace the dot-product score with a log-Gaussian density score. Each slice $g$ is parameterized by a center $\boldsymbol{\mu}_g \in \mathbb{R}^D$ and a positive-definite precision matrix $\mathbf{P}_g \in \mathbb{R}^{D \times D}$. We define the geometry-aware score as:
\begin{equation}
    S_{n,g} = -\frac{1}{2} \|\mathbf{x}_n - \boldsymbol{\mu}_g\|^2_{\mathbf{P}_g} + \frac{1}{2} \log \det(\mathbf{P}_g),
\end{equation}
where $\|\mathbf{x}\|^2_\mathbf{P} = \mathbf{x}^\top \mathbf{P}\mathbf{x}$ denotes the squared Mahalanobis norm. Substituting this into Eq.~\ref{eq:generic_routing} yields weights proportional to (up to a per-$n$ normalization):
\begin{equation}
    A_{n,g} \propto \sqrt{\det(\mathbf{P}_g)} \cdot \exp\left( -\frac{1}{2} \|\mathbf{x}_n - \boldsymbol{\mu}_g\|^2_{\mathbf{P}_g} \right).
\end{equation}
For a fixed $g$, the quadratic form $(\mathbf{x}_n - \boldsymbol{\mu}_g)^\top \mathbf{P}_g (\mathbf{x}_n - \boldsymbol{\mu}_g)$ defines ellipsoidal level sets determined by the eigenspectrum of $\mathbf{P}_g$. This allows the router to \emph{stretch and compress} the receptive field around a point in an explicitly learned, anisotropic manner. The ellipsoids are defined over the latent coordinates, and the encoder decides how these correspond to coordinates in the physical space. While $\mathbf{P}_g$ may be full rank, we adopt a diagonal parameterization:
\begin{equation}
    \mathbf{P}_g = \text{diag} (p_{g,1}, \dots, p_{g, D}), \quad p_{g,i} > 0.
\end{equation}
We ensure positive definiteness by parameterizing $p_{g,i} = \text{softplus}(\hat{p}_{g,i}) + \varepsilon$, where $\varepsilon > 0$. While this choice implicitly forces the ellipsoids to be axis-oriented in the latent coordinates, we show next that these axis-aligned ellipsoids induce receptive fields in the physical coordinates through the pullback that are oriented via the encoder map. 

The term $\frac{1}{2}\log\det(\mathbf{P}_g)$ acts as the Gaussian normalization constant; algebraically, it balances the softmax logits to ensure that slices with higher variance (lower precision) do not artificially dominate the routing weights.

\subsection{Theoretical Connections: Pullback Pseudometrics and Oriented Receptive Fields}

METRO can be interpreted as learning localized kernels in the feature space. For each slice \(g\), the latent Mahalanobis geometry induces a pullback pseudometric on the physical coordinate space through the encoder \(\phi\). We define the pullback bilinear form
\begin{equation}
    \label{eq:pullback}
    g_{\mathrm{pb}}^{(g)}(\mathbf{u})(\mathbf{v},\mathbf{w})
    =
    \left\langle
    J_\phi(\mathbf{u})\mathbf{v},
    \mathbf{P}_g J_\phi(\mathbf{u})\mathbf{w}
    \right\rangle ,
\end{equation}
where \(J_\phi(\mathbf{u})\) is the Jacobian of \(\phi\), \(\mathbf{P}_g\) is the slice precision matrix, and \(\langle\cdot,\cdot\rangle\) is the Euclidean inner product in \(\mathbb{R}^D\). This form is always positive semidefinite, and becomes positive definite when \(J_\phi(\mathbf{u})\) has full column rank. Thus, even when \(\mathbf{P}_g\) is diagonal in latent coordinates, the induced local geometry in physical coordinates can be non-diagonal through the encoder Jacobian, producing oriented receptive fields. We formalize this local statement below.

\begin{proposition}[Local orientation via pullback pseudometric]
\label{prop:orientation of pullback}
Let \(\phi:\mathbb{R}^m\to\mathbb{R}^D\) be twice differentiable, and let \(\mathbf{P}\in\mathbb{R}^{D\times D}\) be a diagonal matrix with strictly positive diagonal entries. Define
\begin{equation}
    q(\mathbf{u})
    =
    \mathbf{r}(\mathbf{u})^\top \mathbf{P}\mathbf{r}(\mathbf{u}),
    \qquad
    \mathbf{r}(\mathbf{u})=\phi(\mathbf{u})-\boldsymbol{\mu}.
\end{equation}
Then, for any \(\mathbf{u}_0\in\mathbb{R}^m\), the Hessian of \(q\) evaluated at \(\mathbf{u}_0\) satisfies
\begin{equation}
    \label{eq:hessian}
    \nabla^2 q(\mathbf{u}_0)
    =
    2J_\phi(\mathbf{u}_0)^\top \mathbf{P}J_\phi(\mathbf{u}_0)
    +
    2\sum_{k=1}^D
    P_{kk}r_k(\mathbf{u}_0)\nabla^2\phi_k(\mathbf{u}_0).
\end{equation}
In particular, if \(\phi(\mathbf{u}_0)=\boldsymbol{\mu}\), i.e. \(\mathbf{r}(\mathbf{u}_0)=\mathbf{0}\), then
\begin{equation}
    \nabla^2 q(\mathbf{u}_0)
    =
    2J_\phi(\mathbf{u}_0)^\top \mathbf{P}J_\phi(\mathbf{u}_0).
\end{equation}
\end{proposition}

In Appendix~\ref{app:metric_routing_theory}, we provide a continuous-kernel interpretation of the Mahalanobis slicing score and its discrete approximation. Proposition~\ref{prop:orientation of pullback} shows that the local second-order geometry of the latent Mahalanobis distance is governed by the pullback form \(J_\phi^\top \mathbf{P}_g J_\phi\) at points mapped to the slice center. Away from the center, the additional Hessian term in Eq.~\eqref{eq:hessian} captures higher-order distortion from the nonlinear encoder. Therefore, axis-aligned ellipsoids in latent coordinates can induce locally oriented receptive fields in physical coordinates. We empirically measure the alignment between the principal axis of this induced local geometry and the local streamline direction in Appendix~\ref{sec:alignment}; the alignment increases systematically with anisotropy, supporting the interpretation that elongated slices adapt to physically meaningful directional features.

\begin{figure}[t]
    \centering
    \includegraphics[width=1.0\linewidth]{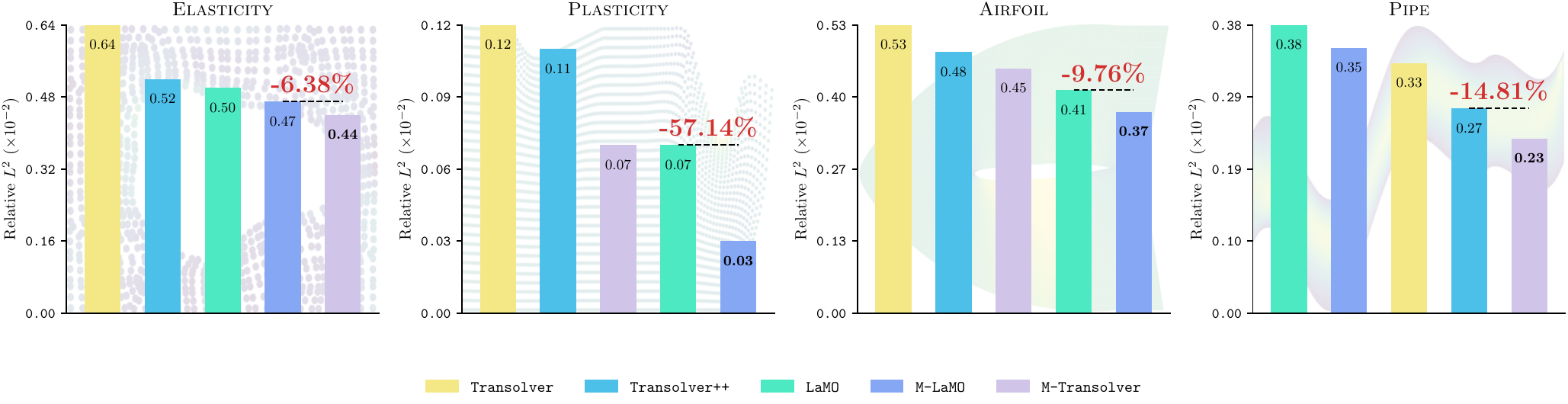}
    \caption{M-Transolver consistently outperforms baselines across complex geometries (Elasticity, Plasticity, Airfoil Pipe), achieving an average \textbf{23.4\% relative error reduction}.}
    \label{fig:four_domains}
\end{figure}

\section{Experiments}

To evaluate the proposed METRO framework, we conduct extensive experiments across diverse physical domains, ranging from standard PDE benchmarks to large-scale industrial design tasks. We demonstrate that integrating our geometry-aware anisotropic slicing into Transolver and LaMO backbones (denoted as \textbf{M-Transolver} and \textbf{M-LaMO}) consistently enhances performance across both structured and unstructured discretizations.

\begin{table*}[t]
\centering
\caption{Quantitative results on industrial benchmarks AirfRANS (unstructured RANS) and AirCraft (3D external aerodynamics). We report relative $L_2$ error ($\downarrow$ lower is better) and Spearman's rank correlation $\rho_L$ ($\uparrow$ higher is better).}
\label{tab:industrial_results}
\begin{sc}
\makebox[\textwidth][c]{% <--- Added this makebox
\resizebox{1.0\textwidth}{!}{%
\begin{tabular}{l|cccc|cccccc}
\toprule
\multicolumn{1}{c|}{\multirow{2}{*}{\textbf{Model}}} & \multicolumn{4}{c|}{\textbf{AirfRANS}} & \multicolumn{6}{c}{\textbf{AirCraft}} \\ \cmidrule(lr){2-5} \cmidrule(lr){6-11} 
\multicolumn{1}{c|}{} & Volume $\downarrow$ & Surf $\downarrow$ & \textbf{$C_L$} $\downarrow$ & \textbf{$\rho_L$} $\uparrow$ & \textbf{$C_p$} $\downarrow$ & \textbf{$\rho$} $\downarrow$ & \textbf{$U$} $\downarrow$ & \textbf{$V$} $\downarrow$ & \textbf{$W$} $\downarrow$ & Pressure $\downarrow$ \\ \midrule

GraphSAGE \citep{hamilton2017inductive} & 0.0087 & 0.0184 & 0.1476 & 0.9964 & 0.1442 & 0.0803 & 0.0148 & 0.2154 & 0.2395 & 0.0923 \\
PointNet \citep{qi2017pointnet} & 0.0253 & 0.0996 & 0.1973 & 0.9919 & 0.1472 & 0.0783 & 0.0147 & 0.2104 & 0.2528 & 0.0937 \\
Graph U-Net \citep{gao2019graph} & 0.0076 & 0.0144 & 0.1677 & 0.9949 & 0.1475 & 0.0784 & 0.0152 & 0.2225 & 0.2477 & 0.0947 \\
MeshGraphNet \citep{pfaff2020learning} & 0.0214 & 0.0387 & 0.2252 & 0.9945 & 0.1372 & 0.0682 & 0.0164 & 0.1835 & 0.2175 & 0.0893 \\
GNO \citep{li2020neural} & 0.0269 & 0.0405 & 0.2016 & 0.9938 & 0.1556 & 0.0785 & 0.0167 & 0.2012 & 0.2318 & 0.1071 \\
Galerkin \citep{cao2021choose} & 0.0074 & 0.0159 & 0.2336 & 0.9951 & 0.0990 & 0.0429 & 0.0110 & 0.1304 & 0.1077 & 0.0480 \\
Geo-FNO \citep{li2023fourier} & 0.0361 & 0.0301 & 0.6161 & 0.9257 & 0.2586 & 0.1227 & 0.0300 & 0.3014 & 0.2757 & 0.1432 \\
GNOT \citep{hao2023gnot} & 0.0049 & 0.0152 & 0.1992 & 0.9942 & 0.0824 & 0.0407 & 0.0092 & 0.1141 & 0.1164 & 0.0436 \\
GINO \citep{li2023geometry} & 0.0297 & 0.0482 & 0.1821 & 0.9958 & 0.1513 & 0.0636 & 0.0160 & 0.1501 & 0.1431 & 0.0853 \\
Transolver \citep{wu2024transolver} & 0.0037 & 0.0142 & 0.1030 & 0.9978 & \textit{0.0748} & \textit{0.0379} & \textit{0.0091} & \textit{0.1107} & \textit{0.1016} & 0.0439 \\
Transolver++ \citep{luotransolver++} & 0.0025 & 0.0084 & 0.1228 & 0.9976 & 0.0784 & 0.0390 & 0.0122 & 0.1318 & 0.1048 & 0.0445 \\
LaMO \citep{tiwarilatent} & \textit{0.0020} & \textit{0.0066} & \textit{0.0977} & \textit{0.9983} & 0.0786 & 0.0400 & 0.0099 & 0.1234 & 0.1094 & \textit{0.0438} \\ \midrule

\textbf{M-Transolver (Ours)} & 0.0014 & 0.0061 & 0.0528 & 0.9991 & \textbf{0.0692} & \textbf{0.0353} & \textbf{0.0060} & \textbf{0.0919} & \textbf{0.0899} & \textbf{0.0389} \\ 

\textbf{M-LaMO (Ours)} & \textbf{0.0010} & \textbf{0.0053} & \textbf{0.0495} & \textbf{0.9994} & 0.0734 & 0.0358 & 0.0075 & 0.1029 & 0.0953 & 0.0405 \\ 
\midrule

\textbf{Relative Promotion} & \textbf{+50.0\%} & \textbf{+19.7\%} & \textbf{+49.3\%} & / & \textbf{+7.5\%} & \textbf{+6.9\%} & \textbf{+34.1\%} & \textbf{+17.0\%} & \textbf{+11.5\%} & \textbf{+11.2\%} \\ \bottomrule
\end{tabular}%
}
} % <--- Closed the makebox
\end{sc}
\end{table*}

\subsection{Experimental Setup}

\textbf{Benchmarks.} We evaluate performance on two categories of datasets. \textbf{Standard PDE Benchmarks:} Following \citep{wu2024transolver, tiwarilatent}, we evaluate on six standard problems: \textit{Airfoil}, \textit{Darcy Flow}, \textit{Navier-Stokes (NS)}, \textit{Plasticity}, \textit{Elasticity}, and \textit{Pipe}. \textbf{Industrial Design Benchmarks:} We employ \textit{AirfRANS} \citep{bonnet2022airfrans}, \textit{AirCraft} \citep{luotransolver++}, and \textit{ShapeNet-Car} \citep{chang2015shapenet}. For AirfRANS and ShapeNet-Car, we additionally measure the Spearman’s rank correlation of lift ($C_L$) and drag ($C_D$) coefficients to assess design utility.
Full quantitative results for ShapeNet-Car and standard PDE benchmarks are provided in Appendix \ref{app:other_experiments}. Detailed dataset specifications are provided in Appendix \ref{sec:benchmarks}.

\textbf{Baselines.} We compare METRO against a comprehensive suite of more than ten baseline methods, including frequency-based operators (e.g., FNO \citep{lifourier}) and state-of-the-art Transformer and State-Space Model operators such as Transolver \citep{wu2024transolver}, Transolver++ \citep{luotransolver++}, and LaMO \citep{tiwarilatent}. For the complete enumeration and configuration of all baselines, we refer readers to Appendix \ref{sec:implementations}. 

% \vspace{-0.6em}

\subsection{Results on Standard Benchmarks}
\label{sec:standard_results}

We first verify the generality of our method by substituting linear slicing with Mahalanobis metrics across standard benchmarks. 
\textbf{Full quantitative results are reported in Table \ref{tab:main_results} (Appendix).} Two key trends emerge from this evaluation. \textit{Enhancement on Irregular Domains:} M-Transolver achieves consistent gains over the original Transolver baseline. Notably, on problems defined over irregular domains (Airfoil, Elasticity, Plasticity, Pipe), our method outperforms all baselines, achieving an average $23.4\%$ relative error reduction.
\textit{Backbone Agnosticism:} Improvements are observed for both M-Transolver and M-LaMO, indicating that anisotropic tokenization complements both Transformer and SSM architectures.

\subsection{Results on Industrial Benchmarks}
\label{sec:industrial_results}

The results on industrial benchmarks (Table \ref{tab:industrial_results}) demonstrate the distinct advantage of our method in handling complex geometries.

1. \textbf{Complex Wake Resolution via Anisotropy (Volume Pressure).} M-LaMO reduces the Volume Pressure error on AirfRANS by nearly 50\% compared to the best baseline. We attribute this improvement to the model's learnable anisotropy. In external aerodynamics, the volumetric domain is dominated by elongated, highly directional structures such as turbulent wakes. Isotropic spherical clusters struggle to cover these features efficiently, often capturing irrelevant free-stream background data. By allowing latent slices to stretch into flow-aligned ellipsoids, METRO efficiently tracks these directional wakes, significantly improving volumetric accuracy.

\textbf{Boundary Layer Isolation via Exponential Locality (Surface \& Lift).}
The Surface error and Lift Coefficient ($C_L$) serve as proxies for resolving the microscopically thin boundary layer. M-Transolver reduces Surface error by 32.2\% and $C_L$ error by 48.7\%. We hypothesize that this is driven by the exponential metric decay induced by the quadratic Mahalanobis distance. In high-dimensional latent spaces, isotropic routing can allocate diffuse neighborhoods that mix near-wall features with less relevant surrounding flow. The quadratic term introduces an exponential metric-locality bias, encouraging surface-focused attention heads to assign lower weight to metric-distant regions and helping preserve the high-gradient signals required for accurate lift prediction.

\textbf{3. Vector Field Sensitivity (AirCraft).}
On the 3D AirCraft dataset, we observe substantial gains in the vertical ($V$, $+17.0\%$) and spanwise ($W$, $+11.5\%$) velocity components. These fields often represent subtle, directional secondary flows (e.g., wingtip vortices). The learnable anisotropy appears to allow the model to align receptive fields with these specific flow directions, rather than averaging them out with a rotationally invariant dot product.

\textbf{4. Practical Utility for Design Optimization.}
Beyond reconstruction, the \textit{Spearman’s rank correlation} ($\rho_L$) for lift reaches $0.9994$ on AirfRANS. In industrial design loops, the ability to correctly \textit{rank} design candidates is often more valuable than absolute numerical precision. This suggests M-Transolver effectively captures the geometric sensitivities required for gradient-based shape optimization. \looseness=-1

\begin{table}[t]
\centering
\caption{Generalization performance on Out-of-Distribution (OOD) tasks for the AirfRANS dataset. We evaluate extrapolation capabilities on unseen Reynolds numbers and Angles of Attack. $\downarrow$ indicates lower is better; $\uparrow$ indicates higher is better (Spearman's rank correlation).}
\label{tab:ood_AirfRANS}
\begin{sc}
\resizebox{0.7\textwidth}{!}{%
\begin{tabular}{l|cc|cc}
\toprule
\multicolumn{1}{c|}{\multirow{2}{*}{\textbf{Model}}} & \multicolumn{2}{c|}{\textbf{OOD Reynolds}} & \multicolumn{2}{c}{\textbf{OOD Angles}} \\
\multicolumn{1}{c|}{} & \textbf{$C_L$} $\downarrow$ & \textbf{$\rho_L$} $\uparrow$ & \textbf{$C_L$} $\downarrow$ & \textbf{$\rho_L$} $\uparrow$ \\ \midrule

Simple MLP & 0.6205 & 0.9578 & 0.4128 & 0.9572 \\
GraphSAGE \citep{hamilton2017inductive} & 0.4333 & 0.9707 & 0.2538 & 0.9894 \\
PointNet \citep{qi2017pointnet} & 0.3836 & 0.9806 & 0.4425 & 0.9784 \\
Graph U-Net \citep{gao2019graph} & 0.4664 & 0.9645 & 0.3756 & 0.9816 \\
MeshGraphNet \citep{pfaff2020learning} & 1.7718 & 0.7631 & 0.6525 & 0.8927 \\ \midrule

GNO \citep{li2020neural} & 0.4408 & 0.9878 & 0.3038 & 0.9884 \\
Galerkin \citep{cao2021choose} & 0.4615 & 0.9826 & 0.3814 & 0.9821 \\
GNOT \citep{hao2023gnot} & 0.3268 & 0.9865 & 0.3497 & 0.9868 \\
GINO \citep{li2023geometry} & 0.4180 & 0.9645 & 0.2583 & 0.9923 \\
Transolver \citep{wu2024transolver} & 0.2996 & \textit{0.9896}& \textit{0.1500}& 0.9950 \\
Transolver++ \citep{luotransolver++} & 0.2897 & 0.9856 & 0.2269 & 0.9914 \\
LaMO \citep{tiwarilatent} & \textit{0.2713}& 0.9887 & 0.1505 & \textit{0.9958}\\ \midrule

\textbf{M-Transolver (Ours)} & \textbf{0.1761} & 0.9950 & \textbf{0.0960} & 0.9977 \\
\textbf{M-LaMO (Ours)} & 0.1988 & \textbf{0.9969} & 0.1015 & \textbf{0.9980} \\

\bottomrule
\end{tabular}%
}
\end{sc}
\end{table}

\subsection{Generalization to Out-of-Distribution Physics}

To assess robustness, we adopt the Out-of-Distribution (OOD) protocol from \citep{wu2024transolver}, evaluating performance on AirfRANS under unseen Reynolds numbers and Angles of Attack.

\textbf{1. OOD Reynolds Numbers: Robustness via Metric Locality.}
We evaluate flow regimes with Reynolds numbers outside the training distribution. M-Transolver achieves a relative $L_2$ error of $0.1761$, a substantial improvement over Transolver ($0.2996$) and Transolver++ ($0.2897$). We hypothesize that this robustness is partly due to the \textit{exponential decay} in learned feature distance. Physically, increasing the Reynolds number reduces the thickness of the boundary layer and introduces high-frequency local turbulence \citep{schlichting2016boundary}. Isotropic routing mechanisms can more easily mix localized high-frequency structures with broader flow patterns seen during training. By introducing a metric-locality bias, our method encourages the model to rely on feature-near interactions that appear more stable across regimes, rather than on dataset-specific long-range correlations that may fail at higher Reynolds numbers. \looseness=-1

\textbf{2. OOD Angles of Attack: Alignment via Anisotropy.}
For extrapolation to unseen geometric orientations, M-Transolver achieves a relative error of $0.0960$, significantly outperforming Transolver ($0.1500$) and LaMO ($0.1505$). We attribute this to the geometric adaptability of the precision matrix $\mathbf{P}_g$. While isotropic kernels are rotationally invariant, they are inefficient at tracking directional features (like wakes) that rotate relative to the grid. To cover a rotated wake, an isotropic kernel must indiscriminately expand its radius. In contrast, our results suggest that the anisotropic mechanism successfully learns to rotate its principal axes in tandem with the flow features. This allows the model to maintain a tightly aligned receptive field around the wake structure, regardless of its global orientation. \looseness=-1

\subsection{Discussion and Ablation Analysis}

Our ablation studies (Table~\ref{tab:combined_ablation}) reveal a consistent trend across both datasets: while removing either the quadratic distance term or anisotropy leads to moderate degradation, the model’s performance remains significantly stronger than the Transolver baseline. However, removing \textit{both} components causes the error to regress close to the baseline, indicating that each geometric mechanism contributes a distinct but complementary inductive bias. We interpret these results through the lens of attention geometry and physical structure.

\textbf{1. Complementary Geometry on AirfRANS.}
AirfRANS features two-dimensional flows dominated by sharp boundary layers and surface-driven pressure distributions. Here, we observe that removing the quadratic term increases \textit{surface} error disproportionately (from $0.0061$ to $0.0090$), while removing anisotropy increases \textit{volume} error (from $0.0014$ to $0.0023$). These trends reflect the dual roles of our metric components. \textit{Anisotropy} enables elliptical slicing aligned with thin shear layers and streamline directions, improving accuracy in modeling internal flow structure. \textit{Quadratic Decay} biases attention toward metric-near features, reducing long-range interference—especially important for accurately predicting near-wall scalar fields like pressure.
Notably, the full model achieves better-than-additive gains: combining both mechanisms yields over $60\%$ improvement in volume and surface errors compared to Transolver. Removing both components collapses this advantage, supporting the importance of their joint use for high-aspect-ratio physics.

\textbf{2. Locality Dominance in AirCraft.}
The AirCraft dataset involves fully three-dimensional geometries with complex topologies (e.g., separate wings, tail, fuselage). In this setting, the quadratic term emerges as crucial for predicting velocity components ($U$, $V$, $W$). This suggests that in irregular 3D domains, a \textit{metric-locality bias} is useful not just for resolution, but for preventing \textit{geometric cross-talk}—ensuring that tokens on the wing do not erroneously attend to the tail simply due to linear alignment. Meanwhile, anisotropy provides incremental benefits by capturing global shape effects such as wing sweep. Together, the two mechanisms contribute orthogonal benefits: directionality shapes the receptive field, while the Mahalanobis norm enforces a learned locality bias.

\begin{table*}[t]
\centering
\caption{Unified Ablation Study on Geometric Slicing Components. We evaluate the contribution of the Quadratic term (Exponential Locality) and Anisotropy (Directionality) across both the \textbf{AirfRANS} (2D) and \textbf{AirCraft} (3D) datasets. The ``w/o All'' variant removes both the quadratic distance term and anisotropic precision structure within our metric-routing implementation, yielding a linear isotropic routing variant; the Transolver reference is reported separately.
}
\label{tab:combined_ablation}
\begin{center}
\begin{small}
\begin{sc}
\resizebox{\textwidth}{!}{%
\begin{tabular}{l|cccc|cccccc}
\toprule
\multicolumn{1}{c|}{\multirow{2}{*}{\textbf{Model Variant}}} & \multicolumn{4}{c|}{\textbf{AirfRANS}} & \multicolumn{6}{c}{\textbf{AirCraft (Relative $L_2$ Error $\downarrow$)}} \\ \cmidrule(lr){2-5} \cmidrule(lr){6-11}
\multicolumn{1}{c|}{} & \textbf{Volume} $\downarrow$ & \textbf{Surf} $\downarrow$ & \textbf{$C_L$} $\downarrow$ & \textbf{$\rho_L$} $\uparrow$ & \textbf{$C_p$} & \textbf{$\rho$} & \textbf{$U$} & \textbf{$V$} & \textbf{$W$} & \textbf{$P$} \\ \midrule
\textbf{M-Transolver (Full)} & \textbf{0.0014} & \textbf{0.0061} & \textbf{0.0528} & \textbf{0.9991} & \textbf{0.0692} & \textbf{0.0353} & \textbf{0.0060} & \textbf{0.0919} & \textbf{0.0899} & \textbf{0.0389} \\ \midrule
w/o Anisotropy & 0.0023 & 0.0066 & 0.0752 & 0.9988 & 0.0702 & 0.0358 & 0.0071 & 0.0987 & 0.0984 & 0.0402 \\
w/o Quadratic Term & 0.0018 & 0.0090 & 0.0780 & 0.9989 & 0.0718 & 0.0363 & 0.0073 & 0.1024 & 0.0984 & 0.0411 \\
w/o All Components & 0.0049 & 0.0105 & 0.1159 & 0.9974 & 0.0791 & 0.0386 & 0.0093 & 0.1125 & 0.0997 & 0.0448 \\ \midrule
\textit{Transolver (Ref.)} & \textit{0.0037} & \textit{0.0142} & \textit{0.1030} & \textit{0.9978} & \textit{0.0748} & \textit{0.0379} & \textit{0.0091} & \textit{0.1107} & \textit{0.1016} & \textit{0.0439} \\ \bottomrule
\end{tabular}%
}
\end{sc}
\end{small}
\end{center}
\end{table*}

\subsection{Case Study: Visualizing Learned Metric Routing}
\label{sec:case_studies}

Having established the quantitative superiority of M-Transolver, we now revisit the geometric hypotheses from Section 1. To examine whether these gains are associated with spatial disentanglement and anisotropic alignment, we visualize the internal representations under three distinct physical regimes. \looseness=-1

\subsubsection{Hypothesis 1: From Isotropic Entanglement to Spatial Partitioning}

Our first premise was that isotropic attention forces \textit{mode collapse}, where multiple slices redundantly attend to the same global features. We test this on the Darcy Flow benchmark (Fig.~\ref{fig:darcy_slices}).

\textbf{Visualizing Disentanglement.} As hypothesized, the standard Transolver exhibits severe spatial redundancy, with the majority of slices collapsing into identical spherical blobs centered on the domain’s highest-energy region. In contrast, M-Transolver breaks this symmetry (Fig.~\ref{fig:darcy_ours}). The learned slices effectively partition the physical domain with minimal overlap, evolving from rigid spheres into flexible geometric primitives—elongated ellipsoids and boundary-hugging contours. This supports the interpretation that the Mahalanobis metric helps the model align with the porous medium structure rather than simply imposing a global grid.

\subsubsection{Hypothesis 2: Robustness in Chaotic Regimes}

We further argued that a metric-locality bias is useful to prevent feature entanglement in high-entropy systems. We examine this using the Navier-Stokes benchmark at the chaotic timestep $T=20$. \looseness=-1

\textbf{Separation vs. Noise.} As fluid dynamics evolve into chaos, the baseline representation degrades into global noise, covering the domain indiscriminately (Fig.~\ref{fig:ns_baseline}). M-Transolver, however, maintains \textit{local separation} even in this highly non-linear regime. The learned metric-locality bias acts as a stabilizer, preventing attention heads from ``leaking'' into the chaotic background and ensuring that individual slices track distinct vortices without entanglement.

\subsubsection{Hypothesis 3: The Reality of Learned Anisotropy}

Finally, we investigate the geometric mechanism by visualizing the learned precision matrix $\mathbf{P}_s$ on the AirCraft dataset. We compare M-Transolver against an ablated variant where the quadratic term ($\mathbf{x}^\top \mathbf{P} \mathbf{x}$) is removed.

\textbf{Heatmap Analysis.} The heatmap for the ablated model (Fig.~\ref{fig:anisotropy_comparison}) displays distinct horizontal banding. This indicates that while the learned precision scales vary \textit{between} slices, they remain uniform \textit{within} each slice ($\mathbf{P}_s \approx \alpha_s \mathbf{I}$).

\textbf{Mathematical Interpretation: The Hierarchy of Expressivity.}
This ablation study establishes a clear hierarchy of geometric adaptation:

\begin{itemize}
    \item \textbf{Baseline (Fixed Linear):} Standard Transolver uses linear attention ($\mathbf{q}^\top \mathbf{k}$) with a fixed global temperature ($1/\sqrt{d}$), forcing all slices to operate at the same scale, limiting flexibility.

    \item \textbf{Ablated (Adaptive Linear):} By removing the quadratic term but keeping the learnable $\mathbf{P}_s$, the model reverts to linear attention ($\mathbf{x}^\top \mathbf{P}_s \boldsymbol{\mu}$), but with a crucial advantage. The ``banding'' behavior confirms it functions as an \textit{Adaptive Temperature} mechanism ($\alpha_s \mathbf{x}^\top \boldsymbol{\mu}$), where distinct slices learn different sensitivities—sharpening focus for fine details or flattening it for global trends. This conceptually aligns with the motivation behind the explicit local adaptive temperature in Transolver++ \citep{luotransolver++}, though our method achieves this adaptivity implicitly through the precision matrix rather than a separate heuristic.

    \item \textbf{Full Model (Anisotropic Metric-Local):} The full model (Fig.~\ref{fig:anisotropy_comparison}) restores the quadratic term ($\mathbf{x}^\top \mathbf{P}_s \mathbf{x}$), closing the geometry from infinite hyperplanes into exponentially localized \textit{ellipsoids}. The high-variance texture in the heatmap indicates that the metric deforms along specific axes. This adds the critical missing capability: \textit{Metric Locality} (suppressing metric-distant noise) and \textit{Directional Alignment} (hugging flow streamlines), which neither the fixed linear baseline nor the adaptive linear ablation can achieve.
\end{itemize}

\section{Conclusion}

In this work, we identified and addressed a geometric mismatch in neural operators: the disparity between predominantly isotropic tokenization mechanisms and the highly anisotropic, irregular structures that arise in many physical systems. We proposed \textbf{Metric-Enhanced Token Routing Operator (METRO)}, a modular framework that reformulates the tokenization process as a learnable anisotropic latent clustering problem. By encouraging \textit{metric locality} via quadratic constraints and \textit{directionality} via learnable anisotropy, METRO generates receptive fields that naturally align with physical structures such as boundary layers and wakes. Furthermore, by effectively shaping the receptive fields to match underlying physical structures, METRO achieves a highly compact latent representation. This token efficiency allows the model to capture complex geometries using significantly fewer latent slices, resulting in both lower error rates and faster training times under strict parameter-matched conditions. Our extensive evaluation on industrial benchmarks demonstrates that these contributions are not merely additive but synergistic: METRO achieves a $43\%$ error reduction averaged across physical metrics on AirfRANS, sets a new state-of-the-art in OOD extrapolation, and consistently enhances performance across both Transformer and State-Space Model backbones.

\bibliography{neurips_2026}
\bibliographystyle{icml2026}

%%%%%%%%%%%%%%%%%%%%%%%%%%%%%%%%%%%%%%%%%%%%%%%%%%%%%%%%%%%%
\newpage

\appendix

\begin{figure*}[t!] % Changed [t] to [t!] to force placement priority
    \centering
    \begin{subfigure}[b]{1.0\textwidth}
        \centering
        \includegraphics[width=\linewidth]{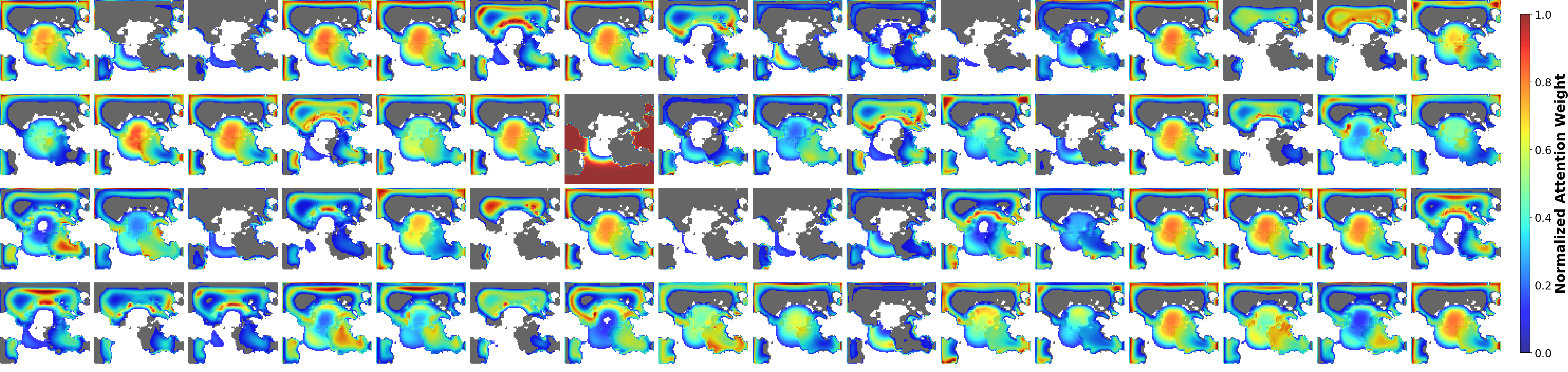}
        \caption{Standard Transolver (Baseline)}
        \label{fig:darcy_baseline}
    \end{subfigure}
    \par\bigskip
    \begin{subfigure}[b]{1.0\textwidth}
        \centering
        \includegraphics[width=\linewidth]{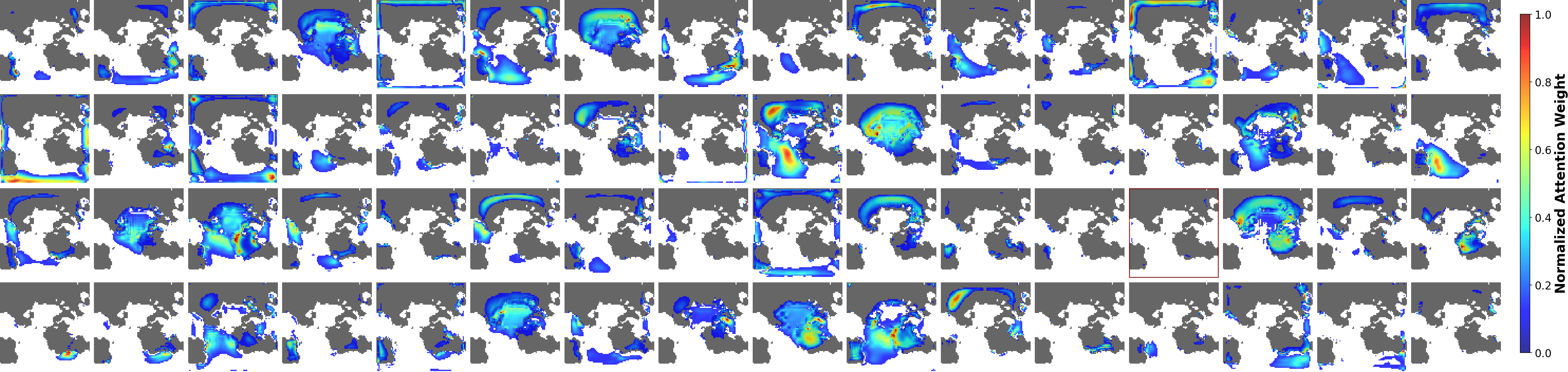}
        \caption{M-Transolver (Ours)}
        \label{fig:darcy_ours}
    \end{subfigure}
    \caption{\textbf{Learned Attention Topology on Darcy Flow.} Visualization of the attention weights for 64 latent slices. \textbf{(a)} The baseline Transolver suffers from high redundancy, with most slices collapsing into identical spherical shapes in the domain center. \textbf{(b)} M-Transolver learns a diverse set of anisotropic shapes (ellipses, boundary layers, curves) that efficiently tile the domain with minimal overlap, demonstrating superior spatial disentanglement.}
    \label{fig:darcy_slices}
\end{figure*}

\begin{figure*}[t!] % Changed [t] to [t!]
    \centering
    \begin{subfigure}[b]{1.0\textwidth}
        \centering
        \includegraphics[width=\linewidth]{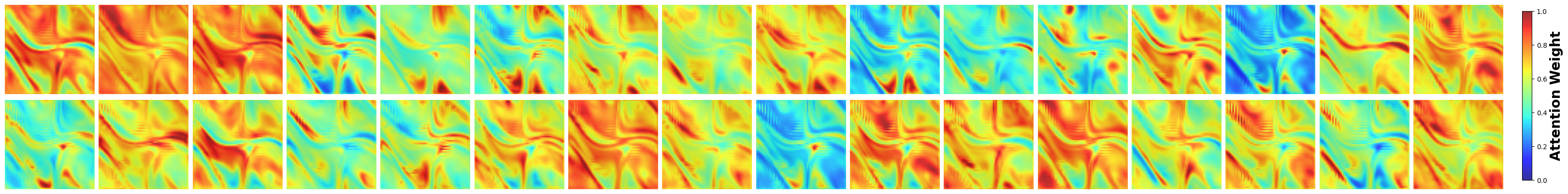}
        \caption{Standard Transolver (Baseline) at T=20}
        \label{fig:ns_baseline}
    \end{subfigure}
    \par\bigskip
    \begin{subfigure}[b]{1.0\textwidth}
        \centering
        \includegraphics[width=\linewidth]{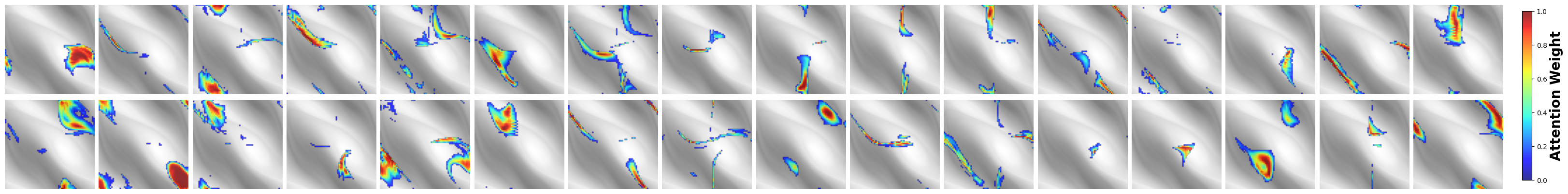}
        \caption{M-Transolver (Ours) at T=20}
        \label{fig:ns_ours}
    \end{subfigure}
    \caption{\textbf{Attention Robustness in Chaotic Flow (Navier-Stokes).} Comparison of learned slice attention at a late timestep ($T=20$). \textbf{(a)} The baseline Transolver slices become highly entangled, attending to the entire domain and losing local specificity. \textbf{(b)} M-Transolver maintains clear separation, with slices focusing on distinct, localized flow features despite the chaotic global dynamics.}
    \label{fig:ns_comparison}
\end{figure*}

\begin{figure}[t]
    \centering
    % One single image file containing both plots side-by-side
    \includegraphics[width=\linewidth]{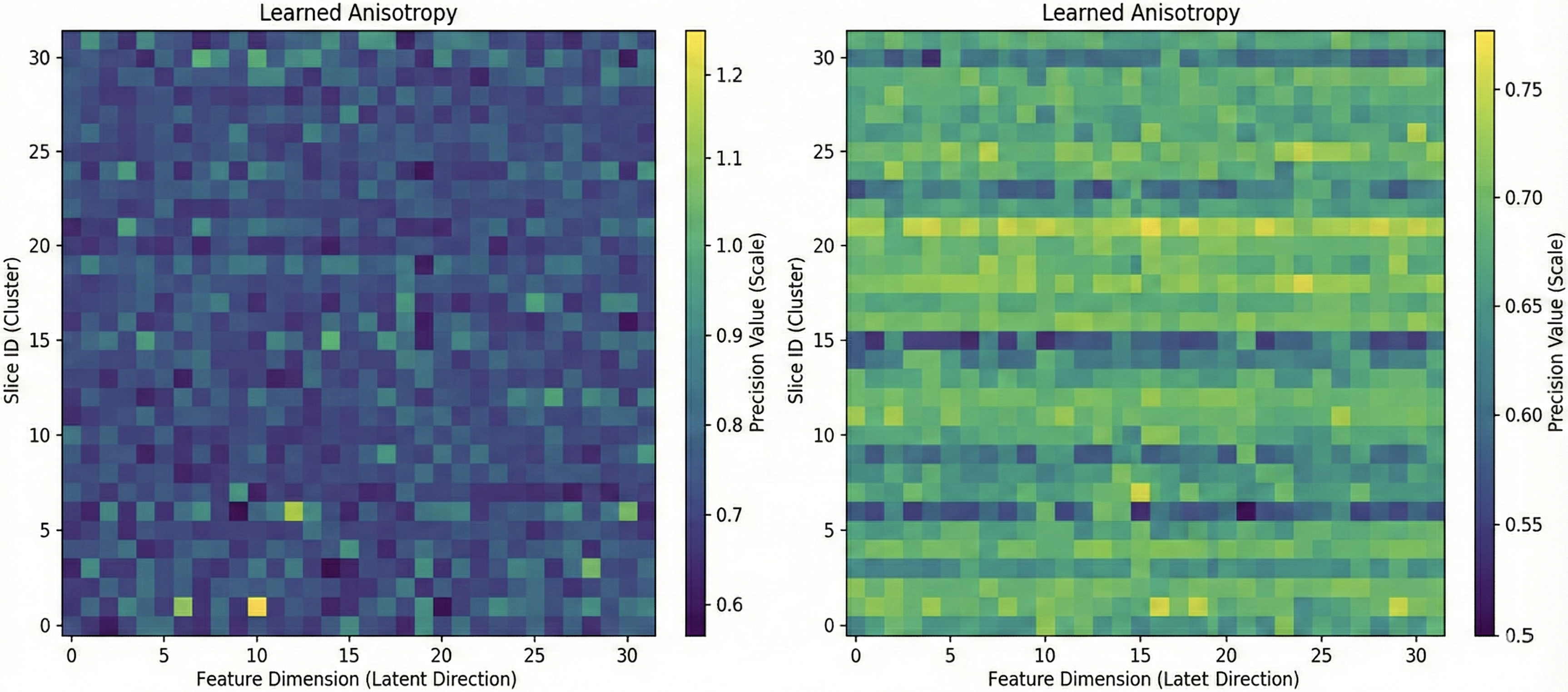}

    \caption{\textbf{Learned Anisotropy ($\mathbf{P}_s$).} (L) The Full M-Transolver captures complex directional variance. (R) Without the quadratic term, the model collapses into isotropic horizontal banding.}
    \label{fig:anisotropy_comparison}
\end{figure}

\section{Other Main Experiments}
\label{app:other_experiments}

In this section, we provide extended quantitative results and detailed analyses that complement the main text. We first extend our evaluation to the domain of automotive aerodynamics using the ShapeNet-Car benchmark, validating the model's robustness on bluff-body flows. We then provide a comprehensive breakdown of the standard PDE benchmarks, offering a granular analysis of the performance trade-offs between architectural paradigms across different geometric topologies.

\begin{table}[H]
\centering
\caption{Quantitative results on the \textbf{ShapeNet-Car} benchmark (3D automotive aerodynamics). We compare against geometric baselines and state-of-the-art neural operators. We report relative $L_2$ error ($\downarrow$) for Volume and Surface fields, and the error/correlation for the Drag Coefficient ($C_D$, $\rho_D$). The best baseline results are \textit{italicized}, and our results are \textbf{bolded}.}
\label{tab:shapenet_results}
\begin{sc}
\resizebox{0.8\textwidth}{!}{%
\begin{tabular}{l|cccc}
\toprule
\multicolumn{1}{c|}{\multirow{2}{*}{\textbf{Model}}} & \multicolumn{4}{c}{\textbf{ShapeNet-Car}} \\ \cmidrule(lr){2-5} 
\multicolumn{1}{c|}{} & Volume $\downarrow$ & Surf $\downarrow$ & \textbf{$C_D$} $\downarrow$ & \textbf{$\rho_D$} $\uparrow$ \\ \midrule

GraphSAGE \citep{hamilton2017inductive} & 0.0461 & 0.1050 & 0.0270 & 0.9695 \\
PointNet \citep{qi2017pointnet} & 0.0494 & 0.1104 & 0.0298 & 0.9583 \\
Graph U-Net \citep{gao2019graph} & 0.0471 & 0.1102 & 0.0226 & 0.9725 \\
MeshGraphNet \citep{pfaff2020learning} & 0.0354 & 0.0781 & 0.0168 & 0.9840 \\
GNO \citep{li2020neural} & 0.0383 & 0.0815 & 0.0172 & 0.9834 \\
Geo-FNO \citep{li2023fourier} & 0.1670 & 0.2378 & 0.0664 & 0.8280 \\
Galerkin \citep{cao2021choose} & 0.0339 & 0.0878 & 0.0179 & 0.9764 \\
GNOT \citep{hao2023gnot} & 0.0329 & 0.0798 & 0.0178 & 0.9833 \\
GINO \citep{li2023geometry} & 0.0386 & 0.0810 & 0.0184 & 0.9826 \\
3D-GeoCA \citep{deng2024geometry} & 0.0319 & 0.0779 & 0.0159 & 0.9842 \\
Transolver \citep{wu2024transolver} & 0.0225 & 0.0789 & 0.0123 & \textit{0.9915} \\
Transolver++ \citep{luotransolver++} & 0.0226 & 0.0800 & 0.0132 & 0.9914 \\
LaMO \citep{tiwarilatent} & \textit{0.0212} & \textit{0.0763} & \textit{0.0119} & 0.9907 \\ \midrule

\textbf{M-Transolver (Ours)} & \textbf{0.0173} & \textbf{0.0741} & \textbf{0.0103} & \textbf{0.9938} \\ 
\midrule

\textbf{Relative Promotion} & \textbf{+18.4\%} & \textbf{+2.9\%} & \textbf{+13.4\%} & / \\ \bottomrule
\end{tabular}%
}
\end{sc}
\end{table}

\subsection{Generalization to Automotive Aerodynamics (ShapeNet-Car)}
\label{app:shapenet}

To validate the robustness of M-Transolver on bluff-body aerodynamics—characterized by massive flow separation and turbulent wakes—we evaluate performance on the \textit{ShapeNet-Car} dataset. Unlike the streamlined flows of aviation airfoils, this task is dominated by drag forces rather than lift.

\textbf{1. Accurate Drag Estimation.}
Estimating the Drag Coefficient ($C_D$) is the primary objective in automotive design optimization. As shown in Table \ref{tab:shapenet_results}, M-Transolver achieves a relative error of $0.0103$, reducing the error of the best baseline (LaMO, $0.0119$) by $13.4\%$. Furthermore, it attains a Spearman's rank correlation of $0.9938$, surpassing Transolver $0.9915$. This high ranking accuracy is critical for design optimization loops, ensuring the model can reliably distinguish between subtle geometric modifications that reduce drag.

\textbf{2. Wake Resolution via Volume Error.}
We observe a substantial improvement in \textit{Volume} error ($+18.4\%$ over LaMO). Automotive simulations involve a large computational domain where the wake region is spatially extensive and highly turbulent. The significant gain in volume accuracy suggests that the Mahalanobis metric effectively focuses attention on these complex wake structures, likely filtering out the uninformative free-stream regions that often dilute the performance of isotropic baselines.

\subsection{Extended Analysis on Standard PDE Benchmarks}
\label{app:standard_benchmarks}

In this section, we present the comprehensive quantitative results and specific experimental configurations for all standard PDE benchmarks evaluated in this study.

As detailed in Table \ref{tab:main_results}, our results demonstrate that the proposed anisotropic routing provides a consistent performance gains across the evaluated benchmark suite. On highly unstructured and irregular datasets, such as \textit{Pipe} (which requires precise modeling of variable boundaries) and the \textit{Elasticity} point cloud, M-Transolver achieves state-of-the-art errors of 0.0023 and 0.0044, respectively. This supports the view that global attention, when augmented with a local anisotropic metric, can capture complex pairwise interactions in disordered systems. Furthermore, the strong performance of M-LaMO on datasets like Plasticity and Airfoil indicates that our tokenization provides backbone-agnostic benefits.

Crucially, this representational advantage extends seamlessly to perfectly regular grids. To ensure a fair architectural evaluation of latent tokenization capacity, we evaluated the state-of-the-art SSM baseline LaMO using the identical isotropic ``slice-and-process'' embedding utilized by Transolver. Under this unified architectural formulation, replacing the isotropic clustering with our Mahalanobis metric allows M-Transolver to significantly outperform all baselines, setting a new state-of-the-art with a \textit{Navier-Stokes} error of 0.0678 and a \textit{Darcy} error of 0.0045. Ultimately, these results demonstrate that introducing a geometry-aware, anisotropic tokenization strategy allows models to efficiently compress both highly structured grids and heavily disordered point clouds into physically meaningful latent representations.

\begin{table}[h]
\centering
\caption{Comparative performance (Relative L2 Error) on standard PDE benchmarks. Datasets are grouped by geometry type.}
\label{tab:main_results}
\begin{sc}
\resizebox{\textwidth}{!}{%
\begin{tabular}{c|l|c|cc|ccc}
\toprule
\multicolumn{2}{c|}{\multirow{2}{*}{\textbf{Operator}}} & \textbf{Point Cloud} & \multicolumn{2}{c|}{\textbf{Regular Grid}} & \multicolumn{3}{c}{\textbf{Structured Mesh}} \\ \cmidrule(lr){3-3} \cmidrule(lr){4-5} \cmidrule(lr){6-8} 
\multicolumn{2}{c|}{} & \textbf{Elasticity} & \textbf{Navier-Stokes} & \textbf{Darcy} & \textbf{Plasticity} & \textbf{Airfoil} & \textbf{Pipe} \\ \midrule

\multirow{4}{*}{\rotatebox{90}{\textbf{Classic}}} 
& U-Net \citep{ronneberger2015u} & 0.0235 & 0.1982 & 0.0080 & 0.0051 & 0.0079 & 0.0065 \\
& ResNet \citep{he2016deep} & 0.0262 & 0.2753 & 0.0587 & 0.0233 & 0.0391 & 0.0120 \\
& Swin \citep{liu2021swin} & 0.0283 & 0.2248 & 0.0397 & 0.0170 & 0.0270 & 0.0109 \\
& DeepONet \citep{lu2021learning} & 0.0965 & 0.2972 & 0.0588 & 0.0135 & 0.0385 & 0.0097 \\ \midrule

\multirow{6}{*}{\rotatebox{90}{\textbf{Frequency}}} 
& WMT \citep{gupta2021multiwavelet} & 0.0359 & 0.1541 & 0.0082 & 0.0076 & 0.0075 & 0.0077 \\
& U-FNO \citep{wen2022u} & 0.0239 & 0.2231 & 0.0183 & 0.0039 & 0.0269 & 0.0056 \\
& FNO \citep{lifourier} & 0.0229 & 0.1556 & 0.0108 & 0.0074 & 0.0138 & 0.0067 \\
& U-NO \citep{rahmanu} & 0.0258 & 0.1713 & 0.0113 & 0.0034 & 0.0078 & 0.0100 \\
& F-FNO \citep{tranfactorized} & 0.0263 & 0.2322 & 0.0077 & 0.0047 & 0.0078 & 0.0070 \\
& LSM \citep{wu2023solving} & 0.0218 & 0.1535 & 0.0065 & 0.0025 & 0.0059 & 0.0050 \\ \midrule

\multirow{8}{*}{\rotatebox{90}{\textbf{Transformer}}} 
& Galerkin \citep{cao2021choose} & 0.0240 & 0.1401 & 0.0084 & 0.0120 & 0.0118 & 0.0098 \\
& HT-Net \citep{liu2022ht} & / & 0.1847 & 0.0079 & 0.0333 & 0.0065 & 0.0059 \\
& OFormer \citep{litransformer} & 0.0183 & 0.1705 & 0.0124 & 0.0017 & 0.0183 & 0.0168 \\
& GNOT \citep{hao2023gnot} & 0.0086 & 0.1380 & 0.0105 & 0.0336 & 0.0076 & 0.0047 \\
& FactFormer \citep{li2023scalable} & / & 0.1214 & 0.0109 & 0.0312 & 0.0071 & 0.0060 \\
& ONO \citep{xiao2024improved} & 0.0118 & 0.1195 & 0.0076 & 0.0013 & 0.0053 & 0.0046 \\
& Transolver \citep{wu2024transolver} & 0.0064 & 0.0900 & 0.0057 & 0.0012 & 0.0053 & 0.0033 \\
& Transolver++ \citep{luotransolver++} & 0.0052 & 0.0719 & 0.0049 & 0.0011 & 0.0048 & 0.0027 \\
& \textbf{M-Transolver (ours)} & \textbf{0.0044} & \textbf{0.0678} & \textbf{0.0045} & 0.0007 & 0.0045 & \textbf{0.0023} \\ \midrule

\multirow{1}{*}{\rotatebox{90}{\textbf{SSM}}} 
& LaMO \citep{tiwarilatent} & 0.0050 & 0.1010 & 0.0064 & 0.0007 & 0.0041 & 0.0038 \\
& \textbf{M-LaMO (ours)} & 0.0047 & 0.1027 & 0.0071 & \textbf{0.0003} & \textbf{0.0037} & 0.0035 \\ \bottomrule

% \multicolumn{2}{c|}{\textbf{Relative Promotion}} & \textbf{+12.0\%} & \textbf{+37.4\%} & -15.4\% & \textbf{+57.1\%} & \textbf{+9.8\%} & \textbf{+9.8\%} \\ \bottomrule
\end{tabular}%
}
\end{sc}
\end{table}

\section{Algorithmic Formulation and Visual Analysis}
\label{sec:appendix_algo}

In this section, we provide the formal procedural details of the M-Transolver layer and visually compare its attention mechanism against the baseline.

\begin{algorithm}[H]
\caption{M-Transolver: Metric-Aware Slicing Mechanism}
\label{alg:m_transolver}
\begin{algorithmic}[1]
\Require Input features $\mathbf{X}\in\mathbb{R}^{N\times D_{\mathrm{in}}}$, number of slices $G$.
\Ensure Output features $\mathbf{Y}_{\mathrm{out}}\in\mathbb{R}^{N\times D_{\mathrm{out}}}$.

\State \textbf{1. Projection}
\State $\mathbf{X}_{\mathrm{geo}} \gets \mathrm{Linear}_{\mathrm{geo}}(\mathbf{X})$
\Comment{$\mathbf{X}_{\mathrm{geo}}\in\mathbb{R}^{N\times D}$}
\State $\mathbf{V} \gets \mathrm{Linear}_{\mathrm{val}}(\mathbf{X})$
\Comment{$\mathbf{V}\in\mathbb{R}^{N\times C}$}

\State \textbf{2. Metric Parameters}
\State Retrieve learnable centers $\boldsymbol{\mu}\in\mathbb{R}^{G\times D}$ and raw diagonal precisions $\mathbf{P}_{\mathrm{raw}}\in\mathbb{R}^{G\times D}$.
\State $\mathbf{P} \gets \mathrm{softplus}(\mathbf{P}_{\mathrm{raw}})+\epsilon_{\mathrm{pd}}$
\Comment{Positive diagonal precisions}

\State \textbf{3. Metric Slice Assignment}
\For{$g=1,\dots,G$}
    \For{$n=1,\dots,N$}
        \State $d^2_{n,g}
        \gets
        \sum_{k=1}^{D}
        P_{g,k}\bigl(X_{\mathrm{geo},n,k}-\mu_{g,k}\bigr)^2$
        \Comment{Diagonal Mahalanobis distance}
        \State $L_{n,g}
        \gets
        -\frac12 d^2_{n,g}
        +
        \frac12\sum_{k=1}^{D}\log P_{g,k}$
        \Comment{Metric logit}
    \EndFor
\EndFor

\State $\mathbf{A}^{\mathrm{row}} \gets \mathrm{softmax}_{g}(\mathbf{L})$
\Comment{Row-wise slice assignments, $\sum_g A^{\mathrm{row}}_{n,g}=1$}

\State \textbf{4. Aggregate into Slice Tokens}
\State $m_g \gets \sum_{n=1}^{N} A^{\mathrm{row}}_{n,g}$ for each $g$
\Comment{Slice masses}
\State $A^{\mathrm{slice}}_{n,g}
\gets
A^{\mathrm{row}}_{n,g}/m_g$
\Comment{Column-normalized aggregation weights}
\State $\mathbf{T}_{\mathrm{slice}}
\gets
(\mathbf{A}^{\mathrm{slice}})^\top \mathbf{V}$
\Comment{$\mathbf{T}_{\mathrm{slice}}\in\mathbb{R}^{G\times C}$}

\State \textbf{5. Process Slice Tokens}
\State $\mathbf{T}_{\mathrm{out}}
\gets
\mathrm{TransformerBlock}(\mathbf{T}_{\mathrm{slice}})$
\Comment{Inter-slice processing}

\State \textbf{6. Broadcast Back to Points}
\State $\mathbf{Y}
\gets
\mathbf{A}^{\mathrm{row}}\mathbf{T}_{\mathrm{out}}$
\Comment{Point-wise reconstruction}
\State $\mathbf{Y}_{\mathrm{out}}
\gets
\mathrm{Linear}_{\mathrm{out}}(\mathbf{Y})$

\State \Return $\mathbf{Y}_{\mathrm{out}}$
\end{algorithmic}
\end{algorithm}

% --- THE FIGURE BLOCK ---
\begin{figure}[H]
    \centering
    \vspace{1em} % Add a little space between algorithm and figure
    
    % Panel A: Transolver
    \begin{subfigure}[b]{0.48\textwidth}
        \centering
        \includegraphics[width=\linewidth]{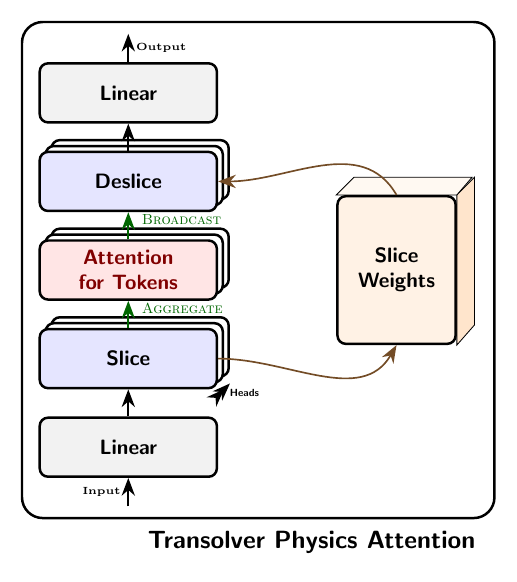}
        \caption{Baseline Transolver (Isotropic)}
        \label{fig:transolver_attn}
    \end{subfigure}
    \hfill
    % Panel B: M-Transolver
    \begin{subfigure}[b]{0.48\textwidth}
        \centering
        \includegraphics[width=\linewidth]{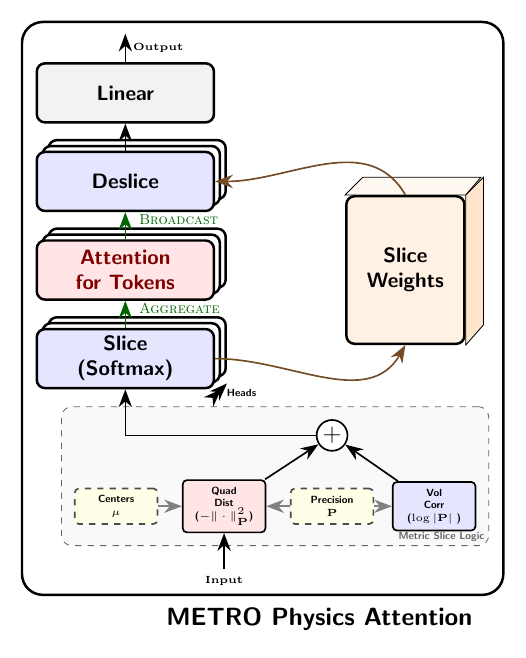}
        \caption{M-Transolver (Anisotropic)}
        \label{fig:malo_attn}
    \end{subfigure}
    
    \caption{\textbf{Impact of Metric Learning on Attention.} 
    (a) The standard Transolver uses isotropic Euclidean distances, treating all directions equally. This often results in attention leaking into irrelevant regions (e.g., across boundary layers). 
    (b) \textbf{M-Transolver} adapts the metric tensor $\mathbf{P}$ to the local physics. The attention mechanism naturally elongates along flow streamlines and compresses across normal gradients, resulting in cleaner, physics-aligned feature aggregation as described in Algorithm \ref{alg:m_transolver}.}
    \label{fig:attention_comparison}
\end{figure}

\section{Continuous Interpretation of Metric Routing}
\label{app:metric_routing_theory}

We give a continuous interpretation of METRO's Mahalanobis metric routing. Rather than treating the full nonlinear architecture as a single linear integral operator, we separate the construction into four parts: (i) metric slice aggregation as a normalized weighted measurement, (ii) the finite-dimensional encoder--processor--decoder form of the full layer, (iii) the linear special case, where the construction reduces to a finite-rank integral operator, and (iv) the local pullback pseudometric induced by the latent Mahalanobis score.

Let \(\Omega\subset\mathbb{R}^m\) be compact and let \(\nu\) be the sampling measure on \(\Omega\). When \(\nu\) has density \(\rho\) with respect to Lebesgue measure, we write \(d\nu(\xi)=\rho(\xi)d\xi\). The distinction is important: without explicit quadrature weights, the continuum limit corresponds to the mesh sampling measure, not necessarily to physical volume measure.

The central object is the metric logit
\[
    \ell_g(\xi)
    =
    -\frac12\|\phi(\xi)-\mu_g\|_{P_g}^2
    +\frac12\log\det P_g,
\]
where \(\phi:\Omega\to\mathbb{R}^D\) is the learned geometry encoder, \(\mu_g\in\mathbb{R}^D\), and \(P_g\succ0\). This logit defines a slice-specific anisotropic notion of proximity in the learned feature space. After converting logits into nonnegative assignment weights, each latent slice computes a normalized metric-weighted measurement of the input field. This corresponds to the column-normalized aggregation weights used for slicing in Algorithm~\ref{alg:m_transolver}.

For the default METRO routing rule,
\[
    w_g(\xi)
    =
    \frac{\exp(\ell_g(\xi))}
    {\sum_{r=1}^G \exp(\ell_r(\xi))},
\]
the weights form a partition of unity across slices:
\[
    w_g(\xi)\ge0,
    \qquad
    \sum_{g=1}^G w_g(\xi)=1.
\]
The aggregation result below does not depend on the softmax form specifically. It applies to any bounded, measurable, nonnegative metric assignment \(w_g\) with positive finite mass, including unnormalized assignments such as \(\exp(\ell_g)\), \(\mathrm{softplus}(\ell_g)\), or \(\sigma(\ell_g)\).

All results in this section should be read conditionally on fixed learned parameters and fixed assignment functions. If the assignments depend on the input field, the resulting operator is generally nonlinear even before applying the latent processor.

\begin{proposition}[Continuous limit of metric slice aggregation]
\label{prop:metric_aggregation_limit}
Let \(\xi_1,\ldots,\xi_N\) be i.i.d. samples from \(\nu\). Let \(f:\Omega\to\mathbb{R}^C\) satisfy \(f\in L^1(\Omega,\nu;\mathbb{R}^C)\). Assume \(w_g:\Omega\to\mathbb{R}_{\ge0}\) is bounded, measurable, and has positive mass
\[
    \alpha_g
    =
    \int_\Omega w_g(\xi)\,d\nu(\xi)
    >
    0.
\]
Define the empirical metric slice coefficient
\[
    z_{g,N}(f)
    =
    \frac{
        \sum_{i=1}^N w_g(\xi_i)f(\xi_i)
    }{
        \sum_{i=1}^N w_g(\xi_i)
    }.
\]
Then
\[
    z_{g,N}(f)
    \to
    z_g(f)
    =
    \frac{
        \int_\Omega w_g(\xi)f(\xi)\,d\nu(\xi)
    }{
        \int_\Omega w_g(\xi)\,d\nu(\xi)
    }
\]
almost surely as \(N\to\infty\), componentwise in \(\mathbb{R}^C\).
\end{proposition}

\begin{proof}
Since \(w_g\) is bounded and \(f\in L^1(\Omega,\nu;\mathbb{R}^C)\), each component of \(w_g f\) is integrable. By the strong law of large numbers, for each component,
\[
    \frac1N\sum_{i=1}^N w_g(\xi_i)f(\xi_i)
    \to
    \int_\Omega w_g(\xi)f(\xi)\,d\nu(\xi)
\]
almost surely. Similarly,
\[
    \frac1N\sum_{i=1}^N w_g(\xi_i)
    \to
    \int_\Omega w_g(\xi)\,d\nu(\xi)
    =
    \alpha_g
\]
almost surely. Since \(\alpha_g>0\), the ratio converges componentwise to the stated limit.
\end{proof}

\paragraph{Deterministic meshes.}
The i.i.d. assumption is used only to state a clean Monte Carlo limit. For deterministic meshes, the same interpretation applies whenever the empirical quadrature rule converges to integration against \(\nu\). With quadrature weights \(q_i\), the corresponding discrete slice coefficient is
\[
    z_{g,N}^{q}(f)
    =
    \frac{
        \sum_{i=1}^N q_i w_g(\xi_i)f(\xi_i)
    }{
        \sum_{i=1}^N q_i w_g(\xi_i)
    }.
\]
If no quadrature weights are used, the limiting measure is the empirical mesh-node distribution.

\begin{proposition}[Metric encoder--processor--decoder form]
\label{prop:metric_epd}
For fixed metric assignment functions \(w_g\) and reconstruction weights \(\widetilde w_g\), the continuous METRO layer can be written as
\[
    \widehat f
    =
    D_{\widetilde w}\circ F_\theta\circ E_w(f),
\]
where the metric encoder is
\[
    (E_w f)_g
    =
    \frac{
        \int_\Omega w_g(\xi)f(\xi)\,d\nu(\xi)
    }{
        \int_\Omega w_g(\xi)\,d\nu(\xi)
    },
\]
the latent processor is \(F_\theta:\mathbb{R}^{G\times C}\to\mathbb{R}^{G\times C'}\), and the decoder is
\[
    (D_{\widetilde w}a)(y)
    =
    \sum_{g=1}^G \widetilde w_g(y)a_g.
\]
Thus, METRO acts through a finite-dimensional set of metric-weighted latent measurements before reconstructing in the output domain.
\end{proposition}

\begin{proof}
By Proposition~\ref{prop:metric_aggregation_limit}, the empirical slice aggregation converges componentwise to the continuous metric-weighted coefficient \((E_w f)_g\). The latent backbone maps the resulting coefficients through \(F_\theta\), and reconstruction applies the weighted decoder \(D_{\widetilde w}\). Composing these three maps gives the stated form.
\end{proof}

If \(F_\theta\) is nonlinear, this is generally not a linear integral operator with an input-independent kernel. Moreover, if the assignment functions themselves depend on the input field, the resulting operator is nonlinear even when \(F_\theta\) is linear.

\begin{proposition}[Linear special case as a finite-rank integral operator]
\label{prop:metric_finite_rank_kernel}
Assume the metric assignments are fixed and the latent processor is linear. In the scalar-channel case, let
\[
    [F_\theta(z)]_g
    =
    \sum_{h=1}^G B_{gh}z_h.
\]
Then METRO reduces to the finite-rank integral operator
\[
    \widehat f(y)
    =
    \int_\Omega K_B(y,\xi)f(\xi)\,d\nu(\xi),
\]
with kernel
\[
    K_B(y,\xi)
    =
    \sum_{g=1}^G\sum_{h=1}^G
    \widetilde w_g(y)B_{gh}
    \frac{
        w_h(\xi)
    }{
        \int_\Omega w_h(s)\,d\nu(s)
    }.
\]
For vector-valued channels, the same formula holds with \(B_{gh}\) replaced by the corresponding channel-mixing matrix.
\end{proposition}

\begin{proof}
Substituting the linear form of \(F_\theta\) into Proposition~\ref{prop:metric_epd} gives
\[
    \widehat f(y)
    =
    \sum_{g=1}^G \widetilde w_g(y)
    \sum_{h=1}^G B_{gh}
    \frac{
        \int_\Omega w_h(\xi)f(\xi)\,d\nu(\xi)
    }{
        \int_\Omega w_h(s)\,d\nu(s)
    }.
\]
Since the sums are finite, we exchange summation and integration to obtain the stated kernel. The vector-valued case is identical with matrix-valued channel maps.
\end{proof}

This finite-rank kernel is a linear special case of the metric-routing formulation, not an assumption on the full nonlinear METRO model.

\begin{proposition}[Role of the log-determinant term]
\label{prop:logdet_metric_volume}
For the metric logit
\[
    \ell_g(\xi)
    =
    -\frac12\|\phi(\xi)-\mu_g\|_{P_g}^2
    +\frac12\log\det P_g,
\]
we have
\[
    \exp(\ell_g(\xi))
    =
    \det(P_g)^{1/2}
    \exp\!\left(
        -\frac12\|\phi(\xi)-\mu_g\|_{P_g}^2
    \right).
\]
Thus, before normalization, the log-determinant term contributes the standard latent Gaussian volume factor associated with the precision \(P_g\).
\end{proposition}

\begin{proof}
Exponentiating the definition of \(\ell_g\) gives
\[
    \exp(\ell_g(\xi))
    =
    \exp\!\left(
        -\frac12\|\phi(\xi)-\mu_g\|_{P_g}^2
    \right)
    \exp\!\left(
        \frac12\log\det P_g
    \right).
\]
Since \(\exp(\frac12\log\det P_g)=\det(P_g)^{1/2}\), the identity follows.
\end{proof}

Under softmax routing across slices, this term changes the relative competition among slices and therefore affects the row-wise assignments \(w_g(\xi)\). For independently normalized non-softmax slice aggregation, however, multiplying a single slice weight \(w_g\) by a positive constant does not change the normalized slice coefficient in Proposition~\ref{prop:metric_aggregation_limit}, because the same constant appears in both numerator and denominator. Therefore, the log-determinant term should be understood primarily as a cross-slice competition and latent-volume factor in the default METRO softmax routing. It should not be interpreted as a physical-domain Jacobian correction unless the encoder has additional structure, such as \(D=m\) and a local diffeomorphism.

\begin{proposition}[Local pullback geometry of metric routing]
\label{prop:metric_pullback_geometry}
Let \(\phi:\mathbb{R}^m\to\mathbb{R}^D\) be twice continuously differentiable and let \(P\succ0\). Define
\[
    q(x)
    =
    \|\phi(x)-\mu\|_P^2.
\]
Then
\[
    \nabla^2 q(x_0)
    =
    2J_\phi(x_0)^\top PJ_\phi(x_0)
    +
    2\sum_{k=1}^D
    [P(\phi(x_0)-\mu)]_k
    \nabla^2\phi_k(x_0).
\]
In particular, if \(\phi(x_0)=\mu\), then
\[
    \nabla^2 q(x_0)
    =
    2J_\phi(x_0)^\top PJ_\phi(x_0).
\]
Thus, near a slice center and under local linearization of the encoder, the physical shape of a latent Mahalanobis slice is governed by the pullback quadratic form \(J_\phi^\top PJ_\phi\).
\end{proposition}

\begin{proof}
Let \(r(x)=\phi(x)-\mu\), so \(q(x)=r(x)^\top P r(x)\). Differentiating once gives
\[
    \nabla q(x)
    =
    2J_\phi(x)^\top P r(x).
\]
Differentiating again gives
\[
    \nabla^2 q(x_0)
    =
    2J_\phi(x_0)^\top PJ_\phi(x_0)
    +
    2\sum_{k=1}^D
    [P r(x_0)]_k
    \nabla^2\phi_k(x_0).
\]
Substituting \(r(x_0)=0\) gives the centered expression.
\end{proof}

Since \(P\succ0\), the pullback form
\[
    Q_g(x)=J_\phi(x)^\top P_gJ_\phi(x)
\]
is positive semidefinite. If \(J_\phi(x)\) has full column rank, then \(Q_g(x)\) is positive definite and defines a local Riemannian metric. If \(J_\phi(x)\) is rank deficient, \(Q_g(x)\) is a positive semidefinite pullback quadratic form, i.e. a local pseudometric. This explains how diagonal latent precision matrices can induce locally oriented anisotropic receptive fields in physical coordinates, while making clear that the statement is local and conditional on the encoder geometry.

% \textbf{Conclusion.}
% Our implementation (Case 3) achieves strict asymptotic equivalence to the standard linear attention (Case 1). Even in the theoretically richer Case 2, the complexity remains linear with respect to the mesh size $N$, ensuring the model scales efficiently to massive physical simulations.
% \end{proof}
\section{Computational Complexity Analysis}

A critical requirement for neural PDE solvers is scalability. Here, we formally prove that the Anisotropic Physics-Attention mechanism maintains linear complexity with respect to the input resolution $N$, regardless of whether a full or diagonal precision matrix is employed.

\begin{proposition}[Asymptotic Complexity]
Let $N$ be the number of mesh points, $G$ be the number of latent slices, and $D$ be the dimension of the feature space. The computational complexity of the Anisotropic Physics-Attention is $\mathcal{O}(N \cdot G \cdot D^k)$, where $k=1$ for the diagonal implementation and $k=2$ for the full matrix formulation. In both cases, the complexity remains linear with respect to the resolution $N$.
\end{proposition}

\begin{proof}
The computational cost is dominated by the scoring function between latent features $\mathbf{z}_j \in \mathbb{R}^D$ and slice centers $\boldsymbol{\mu}_g \in \mathbb{R}^D$. We analyze the complexity for the general and optimized cases.

\textbf{1. Baseline: Standard Linear Attention.}
The standard Transolver employs a dot-product attention mechanism ($\mathbf{q}^\top \mathbf{k}$). For feature vectors of dimension $D$, this requires $D$ multiplications and summations per pair.
\begin{equation}
    \mathcal{C}_{\text{linear}} = \mathcal{O}(N \cdot G \cdot D).
\end{equation}

% \textbf{2. Theoretical Generalization: Full Anisotropy.}
% In the general case, the Mahalanobis distance is computed using a dense precision matrix $\mathbf{P}_g \in \mathbb{R}^{D \times D}$, capturing correlations between all feature dimensions.
% \begin{equation}
%     \text{Score}_{\text{full}} = -\frac{1}{2}(\mathbf{z}_j - \boldsymbol{\mu}_g)^\top \mathbf{P}_g (\mathbf{z}_j - \boldsymbol{\mu}_g) + \frac{1}{2} \log \det(\mathbf{P}_g).
% \end{equation}
% \begin{itemize}
%     \item The matrix-vector multiplication $\mathbf{P}_g (\mathbf{z}_j - \boldsymbol{\mu}_g)$ requires $D^2$ operations.
%     \item Total Complexity: $\mathcal{O}(N \cdot G \cdot D^2)$.
% \end{itemize}
% \textit{Remark:} For tasks operating directly on physical coordinates ($D \in \{2, 3\}$), $D^2$ is a negligible constant, making the full formulation practically as fast as the linear baseline.
\textbf{2. Theoretical Generalization: Full Anisotropy.}
In the general case, the Mahalanobis distance is computed using a dense precision matrix \(P_g\in\mathbb{R}^{D\times D}\), capturing correlations between all feature dimensions:
\[
    \mathrm{Score}_{\mathrm{full}}
    =
    -\frac12(z_j-\mu_g)^\top P_g(z_j-\mu_g)
    +\frac12\log\det(P_g).
\]
The matrix-vector multiplication \(P_g(z_j-\mu_g)\) requires \(O(D^2)\) operations per point-slice pair, giving \(O(NGD^2)\) pairwise scoring cost. Computing or parameterizing \(\log\det(P_g)\) may add an additional per-slice cost, e.g. \(O(GD^3)\) for a naive dense determinant, but this term is independent of the mesh resolution \(N\). Therefore, the full formulation remains linear with respect to \(N\), with dominant pairwise scoring cost \(O(NGD^2)\) in the large-mesh regime.

\textit{Remark:}
For tasks operating directly on low-dimensional physical coordinates (\(D\in\{2,3\}\)), the \(D^2\) pairwise cost and dense log-determinant cost are small constants. In this setting, the full formulation is practically close to the linear baseline, although our high-dimensional latent implementation uses the diagonal formulation for efficiency and stability.

\textbf{3. Efficient Implementation: Diagonal Anisotropy.}
To scale our method to high-dimensional latent spaces (e.g., $D=64$ or $128$), we constrain $\mathbf{P}_g$ to be diagonal, $\mathbf{P}_g = \text{diag}(p_{g,1}, \dots, p_{g,D})$. This models axis-aligned scaling without cross-correlations.
\begin{equation}
    \text{Score}_{\text{diag}} = \sum_{k=1}^D -\frac{1}{2} p_{g,k} (z_{j,k} - \mu_{g,k})^2 + \frac{1}{2} \log(p_{g,k}).
\end{equation}
\begin{itemize}
    \item The term $p_{g,k} (z_{j,k} - \mu_{g,k})^2$ involves element-wise operations only.
    \item This reduces the per-pair cost from $\mathcal{O}(D^2)$ to $\mathcal{O}(D)$.
    \item Total Complexity: $\mathcal{O}(N \cdot G \cdot D)$.
\end{itemize}

\textbf{Conclusion.}
Our implementation (Case 3) achieves strict asymptotic equivalence to the standard linear attention (Case 1). Even in the theoretically richer Case 2, the complexity remains linear with respect to the mesh size $N$, ensuring the model scales efficiently to massive physical simulations.
\end{proof}

\subsection{Parameter-Matched Efficiency.}
To verify that METRO's improvements stem from superior routing rather than added capacity, we evaluated a strict parameter-matched setting. The standard Transolver ($G=64$ slices) contains $\approx 2.82$M parameters and achieves an Elasticity error of $0.0064$ at $\sim 10$s/epoch. By increasing Transolver to $G=96$ slices and reducing M-Transolver to $G=16$ slices, both models are matched at $\approx 2.83$M parameters. Under these constraints, M-Transolver ($0.0048$ error, $\sim 8$s/epoch) simultaneously outperforms the baseline ($0.0061$ error, $\sim 13$s/epoch) in both accuracy and computational speed. The anisotropic routing efficiently covers the physical domain using a fraction of the latent tokens required by isotropic routing.

\begin{table}[t]
\caption{Summary of experiment benchmarks. The first six datasets are from FNO \citep{lifourier} and geo-FNO \citep{li2023fourier}, AirfRANS is from \citep{bonnet2022airfrans}, and AirCraft is a high-fidelity industrial dataset. \#Mesh records the size of discretized meshes. \#Dataset is organized as the number of samples in training and test sets.}
\label{tab:benchmarks}
\begin{center}
\begin{small}
\begin{sc}
\resizebox{\textwidth}{!}{
\begin{tabular}{llccccc}
\toprule
\textbf{Geometry} & \textbf{Benchmarks} & \textbf{\#Dim} & \textbf{\#Mesh} & \textbf{\#Input} & \textbf{\#Output} & \textbf{\#Dataset} \\
\midrule
Point Cloud & Elasticity & 2D & 972 & Structure & Inner Stress & (1000, 200) \\
\midrule
\multirow{3}{*}{Structured Mesh} & Plasticity & 2D+Time & 3,131 & External Force & Mesh Displacement & (900, 80) \\
 & Airfoil & 2D & 11,271 & Structure & Mach Number & (1000, 200) \\
 & Pipe & 2D & 16,641 & Structure & Fluid Velocity & (1000, 200) \\
\midrule
\multirow{2}{*}{Regular Grid} & Navier-Stokes & 2D+Time & 4,096 & Past Velocity & Future Velocity & (1000, 200) \\
 & Darcy & 2D & 7,225 & Porous Medium & Fluid Pressure & (1000, 200) \\
\midrule
\multirow{3}{*}{Unstructured Mesh} & AirfRANS & 2D & 32,000 & Structure & Velocity \& Pressure & (800, 200) \\
 & ShapeNet-Car & 3D & 32,186 & Structure & Velocity \& Pressure & (789, 100) \\
 & AirCraft & 3D & $\sim$300,000 & Structure & 6 Quantities & (110, 10) \\
\bottomrule
\end{tabular}
}
\end{sc}
\end{small}
\end{center}
\end{table}

\section{Latent versus Physical Partition of Unity}
\label{sec:pou_comparison}

Our proposed METRO framework fundamentally operates by constructing a partition of unity over the \textit{latent} feature space $\mathcal{Z}$. This algebraic domain of operation yields a distinct representational advantage over methods that construct localized weighting functions directly in the physical domain $\Omega$.

\textbf{The Constraint of Physical Contiguity.} Standard physical partition of unity methods define weighting functions $\{\psi_k(\xi)\}_{k=1}^K$ over spatial coordinates $\xi \in \Omega$. To maintain operator locality, the support of each function, $\text{supp}(\psi_k)$, is bounded to a contiguous spatial neighborhood defined by a bandwidth $\epsilon$. Algebraically, if two points $\xi_1$ and $\xi_2$ are separated by a Euclidean distance $\|\xi_1 - \xi_2\| > \epsilon$, the condition $\psi_k(\xi_1) > 0 \implies \psi_k(\xi_2) = 0$ is rigidly enforced. Consequently, a physical partition of unity cannot aggregate spatially separated regions exhibiting analogous physical phenomena without expanding its bandwidth to encompass the intervening, potentially dissimilar domain. This spatial constraint forces the network to allocate separate, redundant clusters for repeating physical states across the mesh.

\textbf{Latent Partitioning and Semantic Grouping.} METRO bypasses this limitation by computing the assignment score via the Mahalanobis metric over the learned embeddings $x_n = \phi(\xi_n)$ (Eq. 2). 

To illustrate this, consider the aerodynamic simulation of a vehicle. A physical partition of unity method tasked with resolving localized turbulence, such as the flow separating around the side mirrors, typically requires independent, redundant clusters for the left and right mirrors. Because the physical distance between the mirrors is much larger than the local flow features, a contiguous spatial kernel cannot group them together without also absorbing the unrelated flow fields in between, such as the laminar flow over the windshield.

In contrast, METRO's non-linear encoder $\phi: \Omega \to \mathcal{Z}$ projects regions governed by analogous local dynamics to proximal coordinates in the latent space. Because the flow behavior around both mirrors shares identical semantic features, they are embedded into the same neighborhood in $\mathcal{Z}$. Consequently, a single latent ellipsoid parameterized by $P_g$ encapsulates the features of both structures simultaneously. When projected back to the physical domain via the preimage $\phi^{-1}$, this single latent slice naturally manifests as a union of spatially disjoint subsets, accurately resolving both regions.

This property ensures that a single task-specialized metric $P_g$ can process spatially separated but physically identical structures simultaneously. As illustrated in Figure \ref{fig:pou_comparison_vis}, METRO efficiently aggregates repeating semantic features across the domain, achieving global feature localization while eliminating the parameter redundancy inherent to physically contiguous formulations.

\begin{figure}[htbp]
    \centering
    % Adjust the width parameter (e.g., 0.8\textwidth) as needed for your specific layout
    \includegraphics[width=\textwidth]{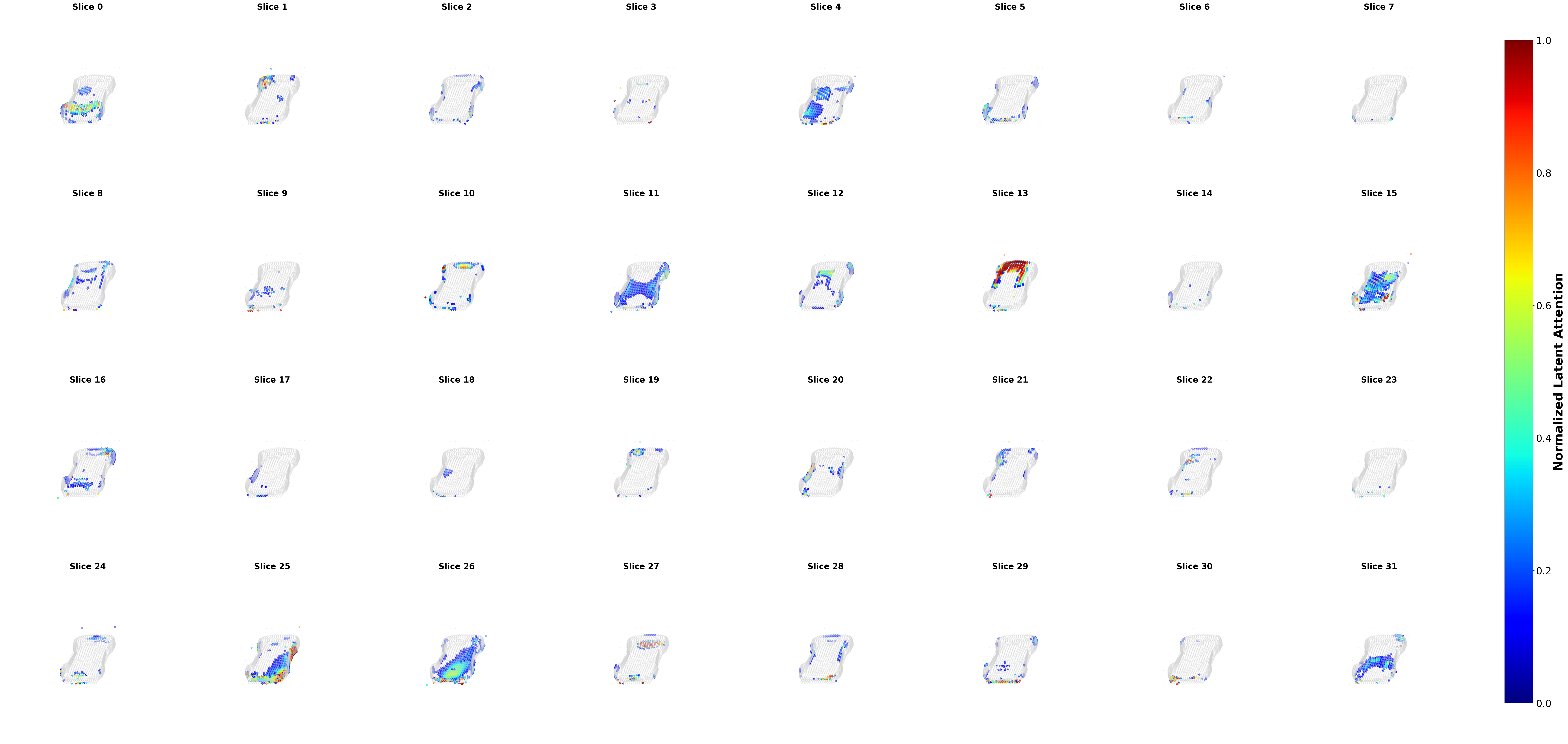}
    \caption{Latent versus physical domain partitioning in vehicle aerodynamics. While standard physical partition of unity methods must allocate independent, redundant spatial clusters to capture spatially disconnected structures—such as the separated flow around bilateral side mirrors—METRO's latent Mahalanobis routing embeds these semantically identical regions into the same latent neighborhood. When projected back to the physical domain via the preimage $\phi^{-1}$, a single latent slice naturally encompasses both mirrors. This demonstrates METRO's ability to achieve efficient global feature localization without being constrained by physical contiguity.}
    \label{fig:pou_comparison_vis}
\end{figure}

\section{Benchmarks}
\label{sec:benchmarks}
To evaluate the effectiveness and generalization capabilities of our model, we conduct experiments across nine distinct benchmarks. These datasets encompass a wide spectrum of physical phenomena, ranging from solid mechanics (Elasticity, Plasticity) to fluid dynamics (Airfoil, Pipe, Navier-Stokes, Darcy, AirfRANS, AirCraft, and ShapeNet-Car). They vary significantly in geometry, dimensionality, and resolution, spanning both standard academic baselines and complex industrial-grade simulations.

\subsection{Solid Mechanics}
\textbf{Elasticity.} This benchmark focuses on estimating the internal stress distribution within hyper-elastic materials governed by the Rivlin-Saunders material model. Each case simulates a unit cell containing a central void, clamped at the bottom and subjected to tension at the top. The geometry is discretized into 972 points. The input is a tensor of shape $972 \times 2$ containing the 2D spatial coordinates, and the target output is a $972 \times 1$ tensor representing the von Mises stress at each point. The dataset comprises 1,000 samples for training and 200 samples for testing.

\textbf{Plasticity.} This task involves predicting the deformation dynamics of a plastic material subjected to impact. The simulation models a scenario where a rigid, frictionless die of arbitrary shape strikes a block at a constant velocity, described by an elastoplastic constitutive model. The input is the geometry of the die, discretized into a $101 \times 31$ mesh. The objective is to predict the material deformation over 20 future time steps, resulting in an output tensor of shape $20 \times 101 \times 31 \times 4$ (representing deformation in four directions). The dataset includes 900 samples for training and 80 for testing.

\subsection{Fluid Dynamics}
\textbf{Airfoil.} Designed to test the modeling of transonic flow, this benchmark utilizes the Euler equations to simulate flow over an airfoil. The domain is discretized into a structured mesh of $221 \times 51$, with shapes derived from deformations of the baseline NACA-0012 airfoil. The model takes the mesh node positions as input and predicts the Mach number field for each mesh point. The dataset consists of 1,000 training samples and 200 testing samples.

\textbf{Pipe.} This benchmark simulates incompressible fluid flow through a pipe channel, governed by the Navier-Stokes equations with a fixed viscosity of $\nu = 0.005$. The geometry involves a pipe of length 10 and width 1, with varying centerline shapes controlled by cubic polynomials. The fluid field is represented by a $129 \times 129$ structured mesh. The input is a tensor of shape $129 \times 129 \times 2$ (spatial coordinates), and the goal is to predict the horizontal velocity component ($129 \times 129 \times 1$). The split is 1,000 samples for training and 200 for testing.

\textbf{Navier-Stokes.} This classic benchmark models viscous, incompressible flow on a 2D unit torus $(0,1)^2$ using the vorticity formulation. With a fixed viscosity of $10^{-5}$, the data is discretized on a $64 \times 64$ regular grid. The task is a temporal forecasting problem: given the fluid velocity field for the past 10 time steps, the model must predict the velocity for the subsequent 10 steps. The dataset provides 1,000 simulations with varied initial conditions for training and 200 for testing. 

\textbf{Darcy Flow.} This dataset simulates steady-state fluid flow through a porous medium on a unit square, governed by Darcy's law (a second-order elliptic equation). The input is the permeability field representing the porous structure, and the output is the resulting fluid pressure field. While originally generated on a $421 \times 421$ grid, it is downsampled to $85 \times 85$ for experimentation. The dataset contains 1,000 training samples and 200 testing samples. 

\subsection{Industrial Applications}
\textbf{AirfRANS.} Providing higher fidelity than the standard Airfoil benchmark, AirfRANS utilizes Reynolds-Averaged Navier–Stokes (RANS) simulations. It includes airfoils from the NACA 4-digit and 5-digit series, discretized into finer meshes of approximately 32,000 points. The dataset captures variations in airfoil shape, angle of attack, and Reynolds number. The primary evaluation targets are the pressure quantities on the volume and surface, as well as the lift coefficients. The dataset includes 800 training samples and 200 testing samples.

\textbf{AirCraft.} This benchmark represents a significant increase in complexity, involving 3D high-fidelity CFD simulations of over 30 distinct aircraft designs. Simulations cover diverse flow conditions, including variations in Mach number and angle of attack. Unlike 2D benchmarks, each aircraft is discretized into a massive unstructured mesh of approximately 300,000 points, where each point is characterized by spatial coordinates $(x,y,z)$ and surface normals. The model predicts six key physical quantities: pressure coefficient $C_p$, fluid density $\rho$, velocity components ($U, V, W$), and pressure $p$. Due to dataset availability, we utilize 110 samples for training and 10 samples for testing.

\textbf{ShapeNet-Car.} Focused on automotive design, this benchmark estimates the drag coefficient for vehicles moving at 72 km/h. The dataset contains 889 samples derived from the ``car'' category of ShapeNet, with 789 used for training and 100 for testing. Each sample is discretized into an unstructured mesh of 32,186 points, where the input features combine mesh position, signed distance function (SDF), and surface normal vectors. The model is tasked with predicting the velocity and pressure fields at every point, which are subsequently used to calculate the vehicle's aerodynamic drag coefficient.

\subsection{Evaluation Metrics}

To comprehensively assess the performance of our model, we employ three distinct metrics. For all benchmarks, we utilize the standard Relative L2 Error to evaluate field reconstruction accuracy. Additionally, for the AirfRANS and ShapeNet-Car benchmarks, we introduce design-oriented metrics—specifically the error in Aerodynamic Coefficients (Lift and Drag) and Spearman's Rank Correlation—to evaluate the model's utility in aerodynamic optimization tasks.

\textbf{Relative L2 Error (Field).}
For all datasets (Elasticity, Plasticity, Airfoil, Pipe, Navier-Stokes, Darcy, AirfRANS, AirCraft, and ShapeNet-Car), we measure the accuracy of the predicted physical fields using the relative root mean square error. Given the ground truth field $u$ and the model prediction $\hat{u}$, the Relative L2 error (L2RE) is defined as:
\begin{equation}
\text{L2RE} = \frac{\| \hat{u} - u \|_2}{\| u \|_2},
\end{equation}
where $\| \cdot \|_2$ denotes the Euclidean norm. This metric provides a direct, normalized quantification of the deviation between the predicted physics states and the high-fidelity simulations.

\textbf{Relative L2 Error (Aerodynamic Coefficients).}
For industrial benchmarks, the utility of a solver is often determined by its ability to predict scalar quantities derived from the fields. We specifically focus on the Lift Coefficient ($C_L$) for AirfRANS and the Drag Coefficient ($C_D$) for ShapeNet-Car. These coefficients are computed by integrating the pressure $p$ and wall shear stress $\tau$ over the object surface $\partial \Omega$:
\begin{equation}
    C_{(\cdot)} = \frac{2}{\rho v^2 A} \int_{\partial \Omega} (p(\xi)\mathbf{n}(\xi) \cdot \mathbf{d} + \tau(\xi) \cdot \mathbf{d}) d\xi,
\end{equation}
where $\rho$ is the fluid density, $v$ is the inlet velocity, $A$ is the reference area, $\mathbf{n}$ is the surface normal, and $\mathbf{d}$ is the projection direction. For Lift ($C_L$), $\mathbf{d}$ aligns with the vertical axis (e.g., $(0,0,1)$); for Drag ($C_D$), $\mathbf{d}$ aligns with the flow direction. We calculate the relative error between the predicted coefficient $\hat{C}$ and the ground truth $C$:
\begin{equation}
\text{L2RE}_{coef} = \frac{\| \hat{C} - C \|_2}{\| C \|_2}.
\end{equation}

\textbf{Spearman's Rank Correlation.}
In industrial design optimization, the relative ranking of different geometric designs is often more critical than absolute numerical precision. To evaluate how well our model preserves the ordering of designs based on their aerodynamic performance, we utilize Spearman's rank correlation coefficient ($\rho$) on the AirfRANS (Lift) and ShapeNet-Car (Drag) test sets. Given $K$ test samples, we compare the rankings of the ground truth coefficients $R(C)$ against the rankings of the predicted coefficients $R(\hat{C})$:
\begin{equation}
\rho = \frac{\text{cov}(R(C), R(\hat{C}))}{\sigma_{R(C)} \sigma_{R(\hat{C})}},
\end{equation}
where $\text{cov}(\cdot)$ denotes covariance and $\sigma$ represents standard deviation. A higher correlation $\rho \in [-1, 1]$ indicates that the model correctly identifies better-performing designs (e.g., higher lift or lower drag), making it a reliable surrogate for design optimization loops.

\subsection{Implementations}
\label{sec:implementations}

As shown in Table \ref{tab:hyperparameters}, all baselines and our proposed models are trained and tested using consistent strategies to ensure fair comparison. 

\textbf{METRO Models Configuration.}
For our proposed METRO models (M-Transolver and M-LaMO), we adopt the same architectural hyperparameters as the original Transolver \citep{wu2024transolver} to ensure that performance gains are attributable to the learned Mahalanobis physics attention rather than increased model capacity. Specifically, the number of channels $D$ is set to 256 for high-dimensional or complex inputs (Navier-Stokes, AirfRANS, ShapeNet-Car, and AirCraft) and 128 for the remaining benchmarks. To balance computational efficiency with expressivity, we set the number of slices $G$ to 32 when $D=256$ and 64 when $D=128$. 

\textbf{Baselines and Reproduction Protocol.}
For the top three leading baselines (Transolver, LaMO, and Transolver++), we executed a comprehensive reproduction study following the strict grid search strategy and hyperparameter settings detailed in the Transolver paper. In all cases, we were able to reproduce the results within the reported standard deviation; consequently, we report the numbers from the original authors to maintain consistency with the literature. For other established baselines, we reuse the results directly from the Transolver, Transolver++, LNO, and LaMO papers.

\textbf{AirfRANS and ShapeNet-Car Settings.}
For the AirfRANS and ShapeNet-Car datasets, we note that the Transolver++ and LaMO papers did not originally report results on these benchmarks. To ensure a fair comparison, we trained these models (along with our METRO models) using the identical settings adopted by Transolver.

\textbf{AirCraft Settings and Baseline Adaptations.} For the AirCraft dataset, we conducted a grid search for all models, sweeping the number of layers in $\{2,4,6,8\}$, number of channels in $\{64,128,256,512\}$, and number of heads in $\{1,2,4,8\}$. We selected the best-performing configuration provided that the total memory consumption remained under a strict budget of 48GB.

\begin{itemize}
    \item \textit{Baselines Implementation:} We implemented GraphSAGE, PointNet, and Graph U-Net following the codebase of Transolver. For GNO, 3D-GeoCA, and GINO, we adopted the official codebases. All models were trained to jointly predict the six available physical fields: pressure coefficient ($C_p$), density ($\rho$), velocity components ($U, V, W$), and static pressure ($P$).

    \item \textit{Efficiency:} The Transolver family and our METRO variants proved highly efficient, utilizing less than 12GB of memory under optimal settings. Conversely, graph-based baselines such as GNO and MeshGraphNet exceeded the 48GB budget even with a reduced hidden size of 64; thus, for these specific baselines, we allowed the memory usage to exceed the 48GB limit to enable comparison.
\end{itemize}

The training configuration for AirCraft (loss functions, epochs, and optimizer) is identical to that of AirfRANS.

\begin{table}[H]
\caption{Training and model configurations for M-Transolver and M-LaMO. Our settings strictly follow the baselines (Wu et al., 2024) to ensure fair comparison. $L_{rL2}$ denotes the relative L2 loss, $L_v$ and $L_s$ represent volume and surface field losses, and $L_g$ is the spatial gradient regularization term. Note that AirCraft adopts the same training configuration as AirfRANS.}
\label{tab:hyperparameters}
\begin{center}
\begin{small}
\begin{sc}
\resizebox{\textwidth}{!}{
\begin{tabular}{lcccccc|cccc}
\toprule
 & \multicolumn{6}{c}{\textbf{Training Configuration}} & \multicolumn{4}{c}{\textbf{Model Configuration}} \\
\cmidrule(lr){2-7} \cmidrule(lr){8-11}
\textbf{Benchmarks} & \textbf{Loss} & \textbf{Epochs} & \textbf{Init LR} & \textbf{Optimizer} & \textbf{Batch} & \textbf{Scheduler} & \textbf{Layers} $L$ & \textbf{Heads} & \textbf{Channels} $D$ & \textbf{Slices} $G$ \\
\midrule
Elasticity & Relative L2 & 500 & $10^{-3}$ & AdamW & 1 & CosineAnnealing & 8 & 8 & 128 & 64 \\
Plasticity & Relative L2 & 500 & $10^{-3}$ & AdamW & 4 & OneCycle & 8 & 8 & 128 & 64 \\
Airfoil    & Relative L2 & 500 & $10^{-3}$ & AdamW & 4 & OneCycle & 8 & 8 & 128 & 64 \\
Pipe       & Relative L2 & 500 & $10^{-3}$ & AdamW & 4 & OneCycle & 8 & 8 & 128 & 64 \\
Navier-Stokes & Relative L2 & 500 & $10^{-3}$ & AdamW & 2 & OneCycle & 8 & 8 & 256 & 32 \\
Darcy      & $L_{rL2} + 0.1L_g$ & 500 & $10^{-3}$ & AdamW & 8 & OneCycle & 8 & 8 & 128 & 64 \\
\midrule
AirfRANS   & $L_v + L_s$ & 400 & $10^{-3}$ & Adam & 1 & OneCycle & 8 & 8 & 256 & 32 \\
AirCraft   & $L_v + L_s$ & 400 & $10^{-3}$ & Adam & 1 & OneCycle & 8 & 8 & 256 & 32 \\
ShapeNet-Car & $L_v + 0.5L_s$ & 200 & $10^{-3}$ & Adam & 1 & OneCycle & 8 & 8 & 256 & 32 \\
\bottomrule
\end{tabular}
}
\end{sc}
\end{small}
\end{center}
\end{table}

\subsection{Experimental Settings for Out-of-Distribution Generalization}

To evaluate the generalizability of our model, we conducted experiments on unseen physical parameters using the AirfRANS dataset. We established two distinct Out-of-Distribution (OOD) settings: one for \textbf{Reynolds numbers} and one for \textbf{Angles of Attack}. In both cases, the training and test sets are constructed with non-overlapping parameter ranges to ensure the model cannot rely on interpolation. Table~\ref{tab:ood_settings} details the specific ranges and sample sizes used for these experiments.

\begin{table}[h]
\centering
\caption{Settings for OOD generalization experiments on AirfRANS. The training and test sets contain disjoint ranges for Reynolds numbers and angles of attack.}
\label{tab:ood_settings}
\begin{small}
\begin{sc}
\resizebox{\textwidth}{!}{
\begin{tabular}{lccccc}
\toprule
& \multicolumn{2}{c}{\textbf{OOD Reynolds}} & \multicolumn{2}{c}{\textbf{OOD Angles}} \\
\cmidrule(lr){2-3} \cmidrule(lr){4-5}
\textbf{Dataset} & \textbf{Range} & \textbf{Samples} & \textbf{Range} & \textbf{Samples} \\
\midrule
Training Set & $[3 \times 10^6, 5 \times 10^6]$ & 500 & $[-2.5^\circ, 12.5^\circ]$ & 800 \\
Test Set     & $[2 \times 10^6, 3 \times 10^6] \cup [5 \times 10^6, 6 \times 10^6]$ & 500 & $[-5^\circ, -2.5^\circ] \cup [12.5^\circ, 15^\circ]$ & 200 \\
\bottomrule
\end{tabular}
}
\end{sc}
\end{small}
\end{table}

\section{Additional Experiments}

\subsection{Ablation on Precision Matrix Structure}

Finally, we investigate the structural parameterization of the precision matrix $\mathbf{P}_g$. While our standard model employs a diagonal precision matrix, theoretical expressivity can be increased by adding low-rank updates (e.g., $\mathbf{P}_g = \mathbf{D}_g + \mathbf{u}\mathbf{u}^\top$) to explicitly model off-diagonal correlations in the latent space. Table \ref{tab:ablation_diagonal} summarizes the results on the Airfoil and Darcy datasets.

\textbf{Trade-off: Expressivity vs. Optimization.}
Contrary to the expectation that higher theoretical expressivity yields better performance, we observe a degradation when introducing off-diagonal terms: the standard Diagonal model ($0.0045$) outperforms both Rank-1 ($0.0047$) and Rank-2 ($0.0048$) updates. We attribute this to two factors derived from our pullback analysis:

\begin{itemize}
    % \item \textbf{Sufficiency of Diagonal Latent Metrics:} As established in Proposition \ref{prop:orientation of pullback}, the effective metric on the physical domain is given by the pullback $J_\phi^\top \mathbf{P}_g J_\phi$. Since the Jacobian $J_\phi$ is generally dense, it naturally mixes coordinates, allowing the receptive field to rotate and shear in physical space even if $\mathbf{P}_g$ is diagonal. In essence, the non-linear encoder $\phi$ learns to orient the latent coordinate system such that the principal axes of the physical features align with the canonical latent axes. Consequently, explicitly modeling off-diagonal terms in $\mathbf{P}_g$ is redundant, as the encoder is already capable of ``diagonalizing'' the local geometry.
    \item \textbf{Effectiveness of Diagonal Latent Metrics:}
    Proposition~\ref{prop:orientation of pullback} shows that the local pullback geometry in physical coordinates is governed by \(J_\phi^\top P_g J_\phi\) near a slice center. Thus, even when \(P_g\) is diagonal in latent coordinates, composition with the encoder can induce oriented receptive fields in physical coordinates. This suggests that diagonal latent precision matrices can already provide useful anisotropic geometry, while explicit off-diagonal updates may offer less practical benefit on these benchmarks than their nominal expressivity suggests.

    \item \textbf{Optimization Complexity:} Introducing low-rank updates creates parameter redundancy. Since rotations can be represented either by modifying the encoder weights (changing $J_\phi$) or by the precision matrix update ($\mathbf{u}\mathbf{u}^\top$), the model faces a non-unique solution space. This likely creates a more complex optimization landscape with additional saddle points, hindering convergence compared to the strictly constrained diagonal formulation.
\end{itemize}
\begin{table}[H]
\centering
\caption{Ablation study on the precision matrix structure. We compare our standard Diagonal formulation against adding low-rank non-diagonal updates on the \textbf{Airfoil} and \textbf{Darcy} datasets. The results show that increasing rank complexity yields no performance gain (and even degradation on Airfoil), suggesting that axis-aligned scaling is sufficient and easier to optimize for these standard benchmarks.}
\label{tab:ablation_diagonal}
\begin{center}
\begin{small}
\begin{sc}
\resizebox{0.5\columnwidth}{!}{%
\begin{tabular}{l|cc}
\toprule
\textbf{Precision Structure} & \textbf{Airfoil} & \textbf{Darcy} \\ \midrule
\textbf{Diagonal (Standard)} & \textbf{0.0045} & \textbf{0.0045} \\
Diagonal + Rank-1 Update & 0.0047 & \textbf{0.0045} \\
Diagonal + Rank-2 Update & 0.0048 & 0.0046 \\ \bottomrule
\end{tabular}%
}
\end{sc}
\end{small}
\end{center}
\end{table}

\subsection{Scalability and Robustness Analysis}
\label{sec:scalability_analysis}

To evaluate the efficiency and stability of the proposed M-Transolver, we conduct a series of scalability experiments. We analyze the model's sensitivity to three critical hyper-parameters: the resolution of the physical slicing ($G$), the size of the training dataset ($N$), and the network depth ($L$).

\subsubsection{Impact of Slice Resolution ($G$)}
We first investigate the impact of the number of slices $G$, which governs the granularity of the geometric prior. Table \ref{tab:ablation_slices} compares the Relative $L_2$ Error (L2RE) of the baseline Transolver against our M-Transolver on the Elasticity and Darcy benchmarks.

\begin{table}[H]
\centering
\caption{Ablation study on the scalability with respect to the number of slices $G$. We compare the performance (Relative $L_2$) of the baseline Transolver and our M-Transolver on the Elasticity (Point Cloud) and Darcy (Regular Grid) datasets.}
\label{tab:ablation_slices}
\begin{sc}
\resizebox{0.75\columnwidth}{!}{%
\begin{tabular}{cc|cc|cc}
\toprule
\multicolumn{2}{c|}{\textbf{Ablations}} & \multicolumn{2}{c|}{\textbf{Transolver L2RE}} & \multicolumn{2}{c}{\textbf{M-Transolver L2RE}} \\ 
\multicolumn{2}{c|}{\textbf{Number of Slices}} & \textbf{Elasticity} & \textbf{Darcy} & \textbf{Elasticity} & \textbf{Darcy} \\ \midrule
\multirow{8}{*}{\rotatebox{90}{\textbf{Number of Slices}}} 
& 1 & 0.0148 & 0.0386 & 0.0221 & 0.0379 \\
& 8 & 0.0071 & 0.0096 & 0.0052 & 0.0054 \\
& 16 & 0.0067 & 0.0067 & 0.0048 & 0.0048 \\
& 32 & 0.0067 & 0.0063 & 0.0049 & 0.0047 \\
& 64 & 0.0064 & 0.0059 & 0.0044 & 0.0045 \\
& 96 & 0.0061 & 0.0055 & 0.0044 & 0.0045 \\
& 128 & 0.0058 & 0.0054 & 0.0043 & 0.0043 \\
& 256 & 0.0054 & 0.0050 & 0.0044 & 0.0043 \\
& 512 & 0.0059 & 0.0056 & 0.0041 & 0.0044 \\ \bottomrule
\end{tabular}%
}
\end{sc}
\end{table}

As shown in Table \ref{tab:ablation_slices}, setting $G=1$ (global attention without slicing) results in high errors for both models, confirming the necessity of the slice-based attention mechanism for physical simulations. As $G$ increases, the error drops significantly. Notably, \textit{M-Transolver consistently outperforms the baseline across all resolutions}, stabilizing at a lower error floor ($\approx 0.43\%$) compared to the baseline ($\approx 0.54\%$) on Elasticity. Furthermore, M-Transolver achieves optimal performance with fewer slices (around $G=64$), suggesting that the metric-aware mechanism captures critical geometry more efficiently than the isotropic baseline.

\subsubsection{Scalability with Training Samples}
Data efficiency is a crucial metric for neural operators, as generating high-fidelity simulation data is computationally expensive. We evaluate model performance across training sizes ranging from 200 to 1000 samples.

\begin{figure}[h]
    \centering
    \includegraphics[width=\textwidth]{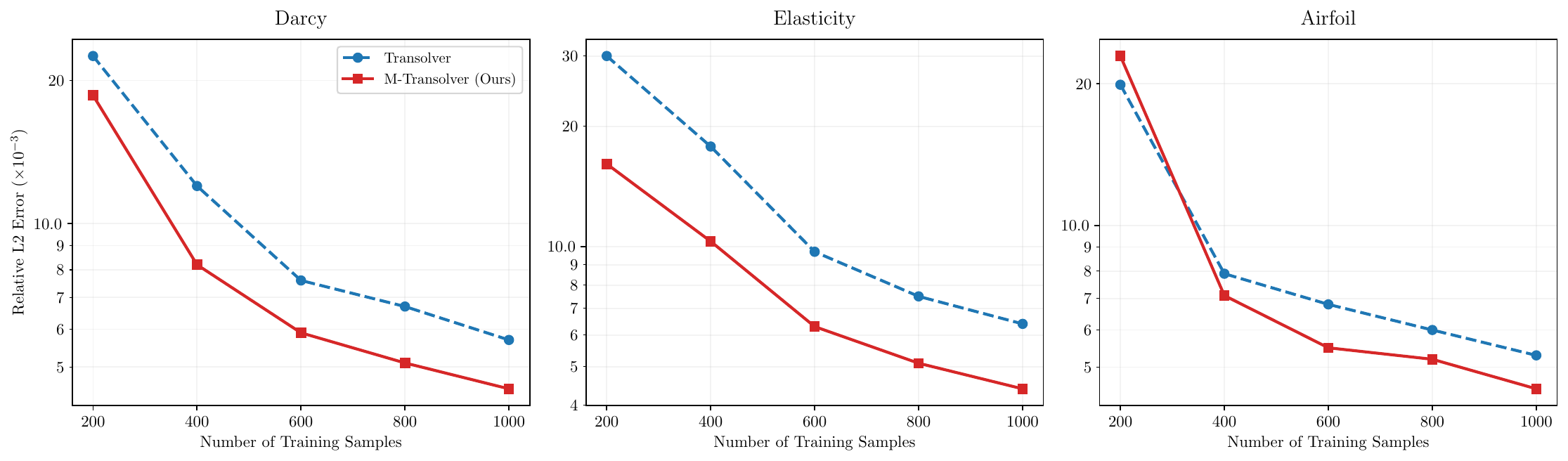}
    \caption{\textbf{Data Efficiency Analysis.} Convergence of Relative $L_2$ Error as the number of training samples increases. M-Transolver (red) exhibits superior data efficiency, achieving significantly lower error rates in the low-data regime (200-400 samples) across all three benchmarks.}
    \label{fig:scalability_samples}
\end{figure}

The results, visualized in Figure \ref{fig:scalability_samples}, demonstrate that M-Transolver exhibits a steeper convergence rate. In the data-scarce regime (200 samples), our method reduces the error by approximately $17\%$ on Darcy ($0.0186$ vs $0.0225$) and $46\%$ on Elasticity ($0.0161$ vs $0.0300$). This indicates that the learned anisotropic geometric prior acts as an effective regularizer, allowing the model to generalize well even with limited supervision.

\subsubsection{Scalability with Network Depth}
Finally, we examine the scalability of the architecture with respect to network depth (number of layers), varying from 8 to 40 layers. This tests the model's ability to propagate physical information without suffering from degradation or over-smoothing.

\begin{figure}[h]
    \centering
    \includegraphics[width=\textwidth]{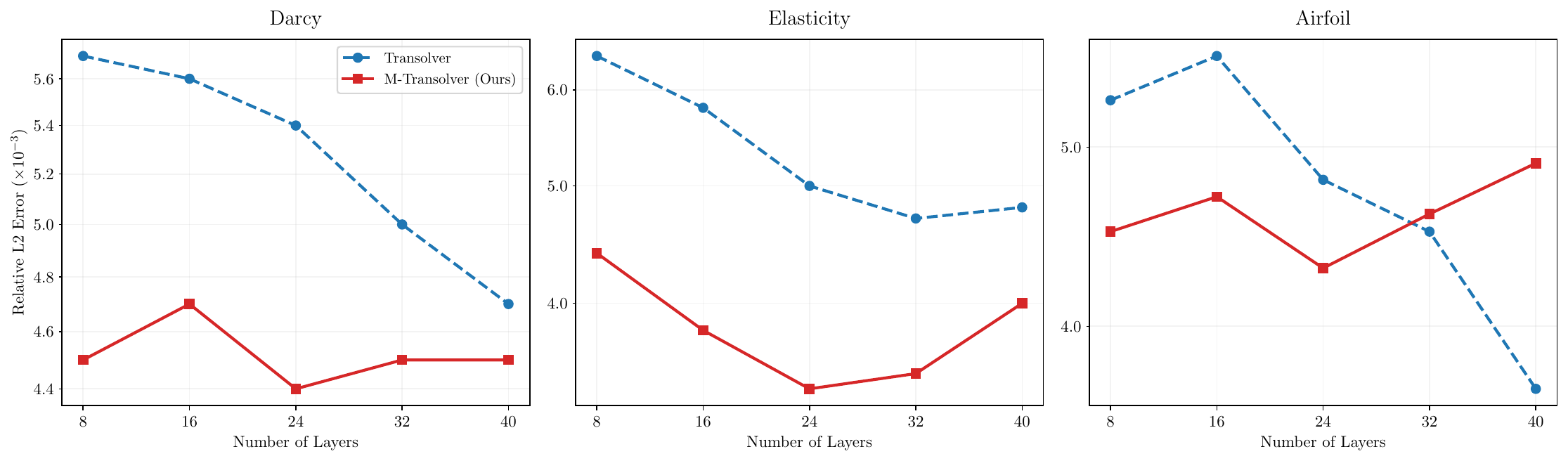}
    \caption{\textbf{Depth Efficiency Analysis.} Performance comparison across varying network depths (8 to 40 layers). M-Transolver demonstrates high stability and robustness, maintaining low error rates even with shallower networks, whereas the baseline often requires greater depth to match performance.}
    \label{fig:scalability_layers}
\end{figure}

As illustrated in Figure \ref{fig:scalability_layers}, M-Transolver is remarkably depth-efficient. On the Elasticity benchmark, our 8-layer model achieves an error of $0.0044$, which outperforms the 40-layer baseline model ($0.0048$). While the baseline Transolver exhibits performance fluctuations at intermediate depths, M-Transolver maintains a stable and consistent error profile. This stability implies that the anisotropic attention mechanism effectively disentangles physical features early in the network, reducing the need for excessive depth to resolve complex flow features.

\subsection{Eccentricity and Alignment of Pullback Metrics}
\label{sec:alignment}
To quantitatively analyze the geometric structure of the learned metrics, we examine the shape and orientation of the attention mechanisms induced by M-Transolver. We define the \textbf{eccentricity} ($e$) of an attention slice based on the eigenvalues ($\lambda_1 \ge \lambda_2$) of its spatial covariance matrix:
\begin{equation}
    e = \sqrt{1 - \frac{\lambda_2}{\lambda_1}}.
\end{equation}
An eccentricity close to $0$ indicates an isotropic (spherical) receptive field, while a value close to $1$ implies a highly anisotropic (needle-like) structure. We further measure the \textit{alignment} as the absolute cosine similarity between the principal axis of the ellipsoid and the local flow velocity vector, where $1.0$ indicates perfect alignment with the streamline.

As shown in Table~\ref{tab:ecc_align}, M-Transolver learns highly anisotropic structures, with the top $10\%$ of slices exhibiting a mean eccentricity of $0.933$. Crucially, these anisotropic slices are not random; they show a high degree of alignment ($0.80$) with the physical flow direction. This suggests that the model naturally discovers the underlying physics, learning to aggregate information along streamlines (where advection dominates) rather than isotropically.

\begin{table}[h]
\caption{Geometric analysis of attention slices in the Navier-Stokes dataset. M-Transolver learns highly anisotropic receptive fields that are strongly aligned with the fluid flow velocity.}
\label{tab:ecc_align}
\begin{center}
\begin{small}
\begin{sc}
\begin{tabular}{lccc}
\toprule
Tier (by Eccentricity) & Mean Ecc. & Mean Alignment & Count \\
\midrule
Top 10\% Most Aniso. & 0.9330 & 0.8000 & 636 \\
Top 25\% Most Aniso. & 0.8952 & 0.7567 & 1591 \\
Top 50\% Most Aniso. & 0.8454 & 0.7265 & 3182 \\
Bottom 10\% (Least)  & 0.4347 & 0.6613 & 636 \\
\bottomrule
\end{tabular}
\end{sc}
\end{small}
\end{center}
\end{table}

\begin{figure}[H] % Use figure* if you need it to span two columns in a two-column paper
    \centering
    \vspace{1em}
    \begin{subfigure}[b]{0.24\linewidth}
        \centering
        \includegraphics[width=\linewidth]{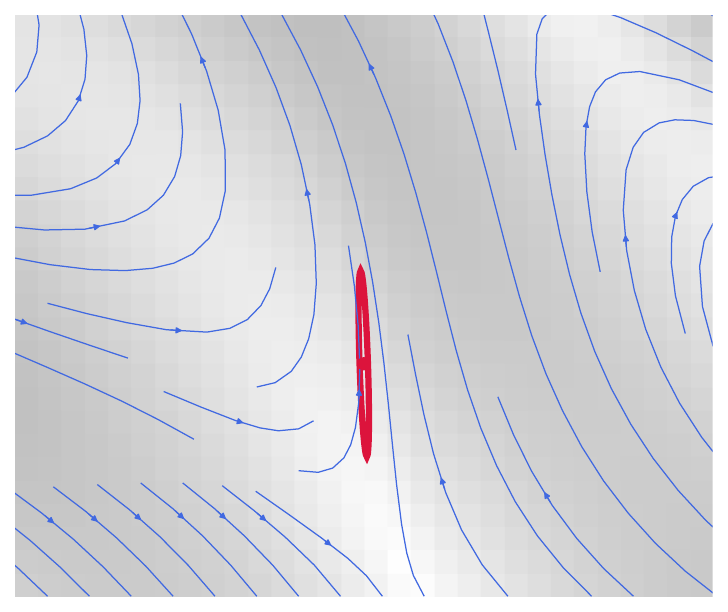}
    \end{subfigure}
    \hfill
    \begin{subfigure}[b]{0.24\linewidth}
        \centering
        \includegraphics[width=\linewidth]{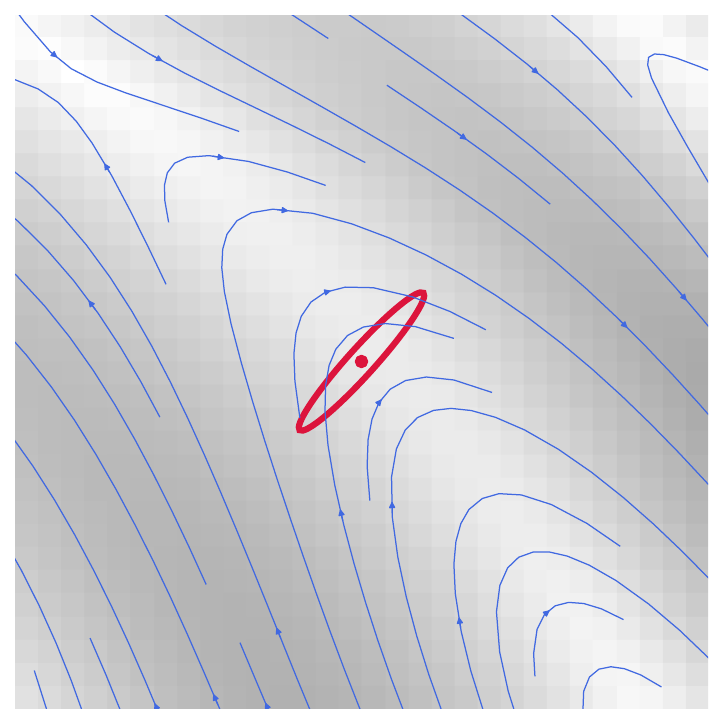}
    \end{subfigure}
    \hfill
    \begin{subfigure}[b]{0.24\linewidth}
        \centering
        \includegraphics[width=\linewidth]{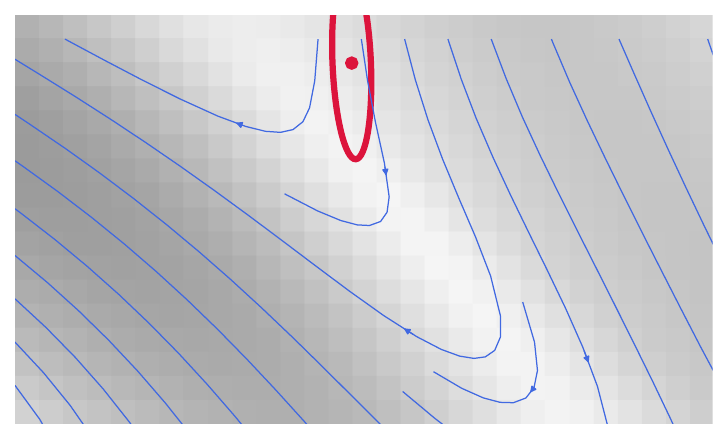}
    \end{subfigure}
    \hfill
    \begin{subfigure}[b]{0.24\linewidth}
        \centering
        \includegraphics[width=\linewidth]{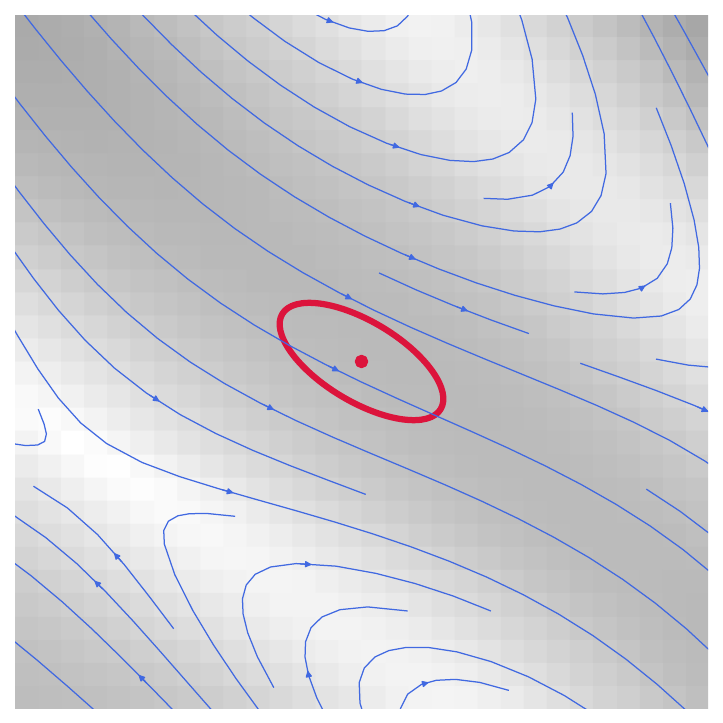}
    \end{subfigure}
    
    \caption{Showcases of the learned pullback metric alignment of M-Transolver. The visualizations demonstrate how the model learns anisotropic receptive fields that align with the physical flow structures.}
    \label{fig:attention_comparison}
\end{figure}

\subsection{Standard Deviations and Robustness}

To ensure statistical rigor, we repeated all experiments across four independent runs and report the mean and standard deviations in Table~\ref{tab:std_dev}. Our proposed method demonstrates robust performance, surpassing the previous state-of-the-art on unstructured domains (Point Cloud and Unstructured Mesh benchmarks) with high confidence. Notably, in complex geometry tasks like ShapeNet-Car and AirfRANS, our method consistently achieves superior correlation coefficients with clear statistical separation from the best baselines.

\begin{table}[H]
\caption{Standard deviations of our model compared to the best previous baselines across all experiments. We report the mean $\pm$ standard deviation over 4 independent runs. Spearman's rank correlations for AirfRANS and ShapeNet-Car are presented as percentages (scaled by 100). All other values are relative errors ($\times 10^{-2}$). The best baseline model for each task is indicated in parentheses.}
\label{tab:std_dev}
\begin{center}
\begin{small}
\begin{sc}
\resizebox{\textwidth}{!}{
\begin{tabular}{lcccccccc}
\toprule
& \multicolumn{3}{c}{Point Cloud} & \multicolumn{2}{c}{Structured Mesh} & Reg. Grid & \multicolumn{2}{c}{Unstructured Mesh} \\
\cmidrule(lr){2-4} \cmidrule(lr){5-6} \cmidrule(lr){7-7} \cmidrule(lr){8-9}
Model & Elasticity & Plasticity & Airfoil & Pipe & Navier-Stokes & Darcy & ShapeNet-Car & AirfRANS \\
\midrule
Best Baseline & $0.50{\scriptstyle\pm 0.021}$ & $0.07{\scriptstyle\pm 0.008}$ & $0.41{\scriptstyle\pm 0.013}$ & $0.27{\scriptstyle\pm 0.026}$ & $7.19{\scriptstyle\pm 0.072}$ & $0.49{\scriptstyle\pm 0.025}$ & $99.15{\scriptstyle\pm 0.01}$ & $99.83{\scriptstyle\pm 0.010}$ \\
& (LaMO) & (LaMO) & (LaMO) & (Transolver++) & (Transolver++) & (Transolver++) & (Transolver) & (LaMO) \\
\midrule
\textbf{Ours} & $\mathbf{0.44}{\scriptstyle\pm 0.015}$ & $\mathbf{0.03}{\scriptstyle\pm 0.006}$ & $\mathbf{0.37}{\scriptstyle\pm 0.017}$ & $\mathbf{0.23}{\scriptstyle\pm 0.028}$ & $\mathbf{6.78}{\scriptstyle\pm 0.17}$ & $\mathbf{0.45}{\scriptstyle\pm 0.01}$ & $\mathbf{99.38}{\scriptstyle\pm 0.025}$ & $\mathbf{99.94}{\scriptstyle\pm 0.010}$ \\
& (M-Transolver) & (M-LaMO) & (M-LaMO) & (M-Transolver) & (M-Transolver) & (M-Transolver) & (M-Transolver) & (M-LaMO) \\
\bottomrule
\end{tabular}
}
\end{sc}
\end{small}
\end{center}
\end{table}

\section{Limitations and Future Work.} While effective, METRO presents limitations that motivate future work. First, the diagonal parameterization of the precision matrix restricts expressivity in complex 3D turbulent flows requiring full off-diagonal correlations. Stabilizing the training of full covariance matrices (e.g., via Riemannian gradient descent) remains an open challenge. Second, METRO learns static spatial metrics $\mathbf{P}_g$, which may struggle to track transient structures in highly unsteady flows, highlighting a need for time-dependent matrices $\mathbf{P}_g(t)$. Finally, METRO operates at a single spatial scale. Integrating it as a boundary-aware frontend for $O(N)$ architectures like LinearNO \citep{hu2025transolver}, or into hierarchical pooling operations like GAOT \citep{wen2025geometry}, is necessary to achieve true multi-scale geometric adaptivity.

\clearpage

\end{document}